\documentclass[final,3p,times,authoryear]{elsarticle}

\usepackage{amssymb}
\usepackage{amsfonts}
\usepackage{amsmath}
\usepackage{amsthm}
\usepackage{algorithm}
\usepackage{algpseudocode}
\usepackage{booktabs}
\usepackage{dirtytalk}
\usepackage{longtable}
\usepackage{hyperref}
\usepackage{multicol}
\usepackage{multirow}
\usepackage{pgfplots}
\usepackage{pgfplotstable}
\usepgfplotslibrary{groupplots}
\pgfplotsset{compat=newest}
\usepackage{plantuml}
\usepackage{subcaption}
\usepackage{sansmath}
\usepackage{tabularx}
\usepackage{tikz}
\usetikzlibrary{fit, positioning}
\usepackage{xcolor}

\newtheorem{theorem}{Theorem}
\newtheorem{proposition}{Proposition}
\newtheorem{remark}{Remark}

\journal{Computers \& Operations Research}

\begin{document}

\begin{frontmatter}




\title{Minimizing the Weighted Number of Tardy Jobs: Data-Driven Heuristic for Single-Machine Scheduling}

\author[a]{Nikolai Antonov\corref{cor1}}
\ead{antonni1@fel.cvut.cz}
\author[b]{Přemysl Šůcha}
\ead{premysl.sucha@cvut.cz}
\author[b]{Mikoláš Janota}
\ead{mikolas.janota@cvut.cz}
\author[b]{Jan Hůla}
\ead{jan.hula@cvut.cz}

\cortext[cor1]{Corresponding author}

\affiliation[a]{organization={Czech Technical University in Prague, Faculty of Electrical Engineering}, 
                addressline={Technická 2}, 
                city={Prague}, 
                postcode={166 27}, 
                country={Czech Republic}}
                
\affiliation[b]{organization={Czech Technical University in Prague, CIIRC}, 
                addressline={Jugoslávských partyzánů 1580/3}, 
                city={Prague}, 
                postcode={160 00}, 
                country={Czech Republic}}

\begin{abstract} 
Existing research on single-machine scheduling is largely focused on exact algorithms, which perform well on typical instances but can significantly deteriorate on certain regions of the problem space. In contrast, data-driven approaches provide strong and scalable performance when tailored to the structure of specific datasets.
Leveraging this idea, we focus on a single-machine scheduling problem where each job is defined by its weight, duration, due date, and deadline, aiming to minimize the total weight of tardy jobs. 
We introduce a novel data-driven scheduling heuristic that combines machine learning with problem-specific characteristics, ensuring feasible solutions, which is a common challenge for ML-based algorithms. 
Experimental results demonstrate that our approach significantly outperforms the state-of-the-art in terms of optimality gap, number of optimal solutions, and adaptability across varied data scenarios, highlighting its flexibility for practical applications.
In addition, we conduct a systematic exploration of ML models, addressing a common gap in similar studies by offering a detailed model selection process and providing insights into why the chosen model is the best fit.
\end{abstract}



\begin{keyword}
scheduling \sep data-driven \sep ML \sep single-machine \sep tardy jobs \sep deadlines



\end{keyword}

\end{frontmatter}



\section{Introduction}\label{sec:Introduction}

This article addresses a strongly NP-hard single-machine scheduling problem \citep{Lawler1983, Yuan2017} of sustained theoretical interest \citep{Hermelin2024} and practical relevance as a component in complex scheduling applications \citep{Sarin2010}. 
The problem essence can be illustrated by a single production line, where the entire work is split into pieces known as \emph{jobs}. 
Due to the technical requirements, we can process only one job at a time, with no interruptions until it is completed.
Every job has two deadlines: a soft \emph{due date}, allowing for penalties if violated, and a hard deadline, which must be strictly met to avoid a production halt.
The objective is to complete all jobs by their hard deadlines while minimizing the total penalty from due date violations. 
In scheduling notation \citep{Graham1979}, this problem is denoted as $1|\tilde{d}_i|\sum w_i U_i$.

Traditional exact approaches, such as integer linear programming (ILP) and dedicated branch-and-bound algorithms \citep{Hejl2022, Baptiste2010}, face scalability and performance limitations under specific data distributions.
These challenges motivate us to find a data-driven solution that leverages machine learning (ML) to enhance the efficiency of scheduling in large and diverse instances.
We introduce novel data-driven scheduling heuristic to minimize the weighted number of tardy jobs on a single machine, building upon~\cite{Antonov2023}.
The key contributions of our work are as follows:
\begin{itemize}
\item
We address a common challenge in combinatorial optimization where ML-based methods often struggle to guarantee feasible solutions~\citep{Bengio2021}. Our approach is specifically designed to always produce a feasible solution when one exists, effectively overcoming this limitation in the context of our scheduling problem.
\item
We provide valuable insights into ML model selection for the studied problem. Building on our previous work, we introduce new, effective representative features and perform a comprehensive evaluation of various ML architectures. The resulting model trains efficiently on small datasets and generalizes well to larger instances. We also offer practical guidance for selecting suitable models tailored to the considered problem.
\item
We evaluate our heuristic using both uniformly distributed data, which is a common practice in operations research, and more realistic datasets where job parameters follow normal, log-normal, and exponential distributions. This broadens the applicability of our approach, and the experiments demonstrate significant improvements in flexibility over the existing methods.
\item
Our heuristic consistently outperforms the state-of-the-art approaches \citep{Baptiste2010, Antonov2023} in solution quality within the same time limits, achieving a significantly smaller optimality gap and obtaining optimal solutions in 80-100\% of cases.
\end{itemize}

The paper is organized as follows.
Section~\ref{sec:Literature Review} provides a review of relevant literature.
Section~\ref{sec:Problem Statement} presents the problem statement.
Section~\ref{sec:Solution Approach} describes our proposed solution approach, including all pertinent details.
Section~\ref{sec:ML Methodology} discusses feature generation and the machine learning models employed.
Section~\ref{sec:Experimental Results} reports on the conducted experiments and compares our approach with the state-of-the-art.
Finally, Section~\ref{sec:Conclusion} concludes the paper and outlines directions for future research.

\section{Literature Review}\label{sec:Literature Review}
As both scheduling and machine learning domains are related to our problem, we split the review into two parts: one dedicated to scheduling and another to data-driven approaches, specifically focusing on machine learning.

\subsection{Related Scheduling Approaches}

Early studies on the problem date back to~\cite{Hariri1994}, where the authors introduced a branch-and-bound algorithm handling up to 300 jobs within one hour.~With advancements in ILP solvers, the size of solvable instances has progressively increased. Based on our experience with modern solvers, an ILP formulation for the considered problem reliably works for small instances up to 500 -- 1500 jobs. However, the number of constraints in the formulation grows quadratically with problem size, leading to slow model construction and memory overflow issues.

The state-of-the-art exact approach is the branch-and-bound algorithm proposed by~\cite{Baptiste2010}, which is capable of solving up to 30000 jobs within an hour. While memory overload is not an issue, the performance of the algorithm is not reliable -- its efficiency varies depending on the dataset. The algorithm significantly deteriorates for specific instance classes, as shown in~\citep{Baptiste2010, Hejl2022} and confirmed in our tests on a distinct instance class (see Section~\ref{subsec:further_discussion}). For example,~\cite{Baptiste2010} observed that their algorithm struggled with approximately 3\% of instances with only 250 jobs, particularly those with linearly correlated job durations and weights. \citet{Hejl2022} refined this algorithm, achieving improved performance on correlated instances, scaling up to 5000 jobs per hour; however, this improvement remains limited to linearly correlated cases. 

Considering another limitation, the algorithm proposed by~\cite{Baptiste2010} demonstrates the highest performance when job parameters follow a uniform distribution. While this is a common assumption in operations research, it does not reflect realistic scenarios. For example, in surgery scheduling, the durations of urgent surgeries often follow a Poisson distribution, whereas the durations of elective surgeries tend to follow a log-normal distribution~\citep{Vanessen2012}. Additional examples are provided in Section~\ref{subsec:datasets}. We observed that under these more realistic conditions, the algorithm consistently fails to solve a certain percentage of small- and mid-sized instances (detailed results can be found in Sections~\ref{subsec:results} and~\ref{subsec:further_discussion}).

In summary, these findings emphasize the need for adaptive strategies -- particularly effective data-driven heuristics -- to address computational challenges across diverse datasets. This perspective aligns with the No-Free-Lunch theorem, which states that no single algorithm performs best across all possible instances. In real-world settings that involve large and heterogeneous problems, scalable heuristic and data-driven methods often offer the most practical and robust alternative.

Only a limited number of studies have explored heuristic algorithms targeting the $1|\tilde{d}_i|\sum w_i U_i$ problem, primarily due to the historical emphasis on exact optimization methods~\citep{Muminu2014}. Among these, metaheuristic approaches have been considered, particularly by \citet{Sevaux2003}, who addressed the closely related problem of minimizing the weighted number of tardy jobs on a single machine without enforcing strict deadlines. However, the existing literature on metaheuristic methods for our specific scheduling problem remains notably outdated, though we can indicate the Honey Badger metaheuristic by~\citep{Hashim2022}, which has shown promising results across diverse scheduling problems~\citep{Hassan2024}.

Traditional rule-based heuristics such as \emph{Earliest Deadline First (EDF)}, \emph{Earliest Due Date first (EDD)}, and \emph{Apparent Tardiness Cost (ATC)} are simple to implement but typically result in substantial optimality gaps in practical scenarios~\citep{Antonov2023}. Notably, among these heuristics, only \emph{EDF} guarantees the feasibility of a solution by ensuring all deadlines are met, provided such a feasible solution exists~\citep{Pinedo2012}.

An effective heuristic algorithm is presented in~\citep{Baptiste2010}. Initially, the authors propose solving a max-profit flow relaxation of the original problem. The obtained solution is then transformed into a feasible solution of the original problem using the dominance rule and ILP (for details on the dominance rule, refer to Section~\ref{subsubsec:cond_probs}). This approach stands as the state-of-the-art heuristic for our problem, and we compare its performance with our method in Section ~\ref{subsec:results}.

\subsection{Related Data-driven Approaches}\label{subsec:related_data-driven_approaches}
The first data-driven applications for scheduling can be traced back to \citep{Franz1989}, which defines the core principles of a data-driven method focusing on the nurse scheduling problem. Subsequently, data-driven methods have been rapidly and successfully employed to address scheduling challenges in various domains, including transportation \citep{Abdelghany2024}, energy supply \citep{Darvazeh2024}, and industry \citep{Rossit2019, Liao2019, Awada2021}, demonstrating significant potential for further integration into the field.

Our approach is inspired by the work of~\cite{Bengio2021}, who highlighted the advantages of using machine learning for various combinatorial optimization problems.
However, the authors raise a critical related challenge: ensuring that the solutions are feasible.
This concern is supported by~\cite{Dias2019}, who discuss the integration of scheduling and planning.
As machine learning is used in our work, we face this challenge as well, tackling it in Section~\ref{subsec:feas_algo} by presenting a framework designed to consistently achieve feasible solutions.

Traditional machine learning methods, such as KNN (K-Nearest Neighbors), SVM (Support Vector Machines), decision trees, or shallow neural networks, have been successfully applied to scheduling problems due to their simplicity, broad applicability, and fast performance. These advantages are illustrated in studies on power supply optimization \citep{Saxena2024, YangZhangYang2022}, with the latter showing the benefits of combining the KNN algorithm and the SVM classifier over more complex neural networks such as LSTM (Long Short-Term Memory). A similar idea appears in \citep{Modos2016}, where KNN and decision trees are used to estimate unknown time–quality trade-off curves, enabling adaptive control of response times in overloaded systems. In general, traditional ML has been employed in two complementary ways. First, it can be integrated directly into the scheduling process. For instance, decision trees have been applied to predict uncertain activity durations in multi-mode project scheduling \citep{Portoleau2024}; ensembles of trees have been trained to learn effective dispatching rules in flexible job-shop and permutation flow-shop problems \citep{Muller2022, Wang2017}; and shallow neural networks have been used to predict task-to-line assignments in production line matching \citep{YangFengGuan2022}. KNN and SVM classifiers have also demonstrated competitive performance in power supply optimization \citep{YangZhangYang2022}. Second, ML can operate at the meta-level, supporting algorithm selection rather than producing schedules directly. In this line, \citet{Chu2023} develop an SVM-based meta-learning framework that recommends the most suitable metaheuristic for the multi-mode RCPSP, while \citet{Uzunoglu2025} apply a learning-to-rank model to select constructive strategies for serial-batch scheduling, with subsequent improvement through tailored genetic algorithms. While our work also integrates ML into the scheduling process, it differs from earlier studies by using classifiers to predict early/tardy job labels. This application directly addresses the central combinatorial challenge of the $1|\tilde{d}_i|\sum w_i U_i$ problem and enables subsequent refinement through ILP and feasibility checks (see Sections~\ref{subsec:ml_models} and \ref{subsec:training}).

The application of deep learning algorithms represents a significant area of data-driven methods in scheduling. Based on recent literature, we identify five major trends in how machine learning is employed: (i) learning dispatching rules and heuristics from data \citep{Jun2020, Janssens2006}; (ii) applying reinforcement learning (RL) for scheduling, including policy construction \citep{Monaci2024, Wu2024, YuanWang2024, Brammer2022, Heger2021, Wang2023}, metaheuristic control \citep{Alicastro2021, Zhang2012}, and multi-agent coordination \citep{LiuPiplaniToro2023}; (iii) accelerating classical optimization procedures, such as branch-and-price, using predictive models \citep{Koutecka2024, Vaclavik2018, Delgoshaei2016}; (iv) estimating objective values via surrogate models \citep{Bouska2022}; and (v) employing attention-based architectures for task selection and schedule modeling \citep{Du2024}.

These approaches also differ in how they address feasibility. In strands (i), (iii), and (iv), feasibility is typically preserved through integration with classical optimization methods -- this is also the case in our work. RL-based approaches (ii) often rely on action masking, reward shaping, or learning implicit constraints through policy training, although feasibility violations can still occur. While feasibility is an important aspect, deep learning methods also face another significant challenge in the context of single-machine scheduling: they often require large training datasets and incur significant training and inference costs. This limitation is addressed in Section~\ref{subsec:ml_models}, where we discuss the application of attention mechanisms \citep{Attention2017}.

\section{Problem Statement}\label{sec:Problem Statement}
We begin by introducing the necessary notations and definitions. Consider a machine (system) capable of performing work divided into pieces, referred to as \emph{jobs}. The machine operates under three basic assumptions: it processes one job at a time, does not interrupt a started job and does not idle -- once a job is completed, the machine immediately starts the next one, continuing until all assigned jobs are finished.

We are given a set of jobs, denoted as \( N = \{1, 2, \dots, n\} \), where each job \( i \in N \) is characterized by its duration \( p_i \), due date \( d_i \), and deadline \( \tilde{d_i} \). These values are positive real numbers, with the constraint \( p_i \leq d_i \leq \tilde{d_i} \) for all \( i \in N \). Additionally, each job \( i \in N \) has a weight (or cost) \( w_i \), representing the job's value. All jobs are available at time 0.

The jobs are processed according to a permutation \( s \) of \( N \), and the completion times of the jobs are denoted as \( C_i^s \) for \( i \in N \). In scheduling terminology, \( s \) represents a \emph{schedule}. We define the set of \emph{early} jobs \( E_s = \{ i \in N \mid C_i^s \leq d_i \} \), which are completed before their due dates, and the set of \emph{tardy} jobs \( T_s = \{ i \in N \mid d_i < C_i^s \leq \tilde{d_i} \} \), which are completed after their due dates but before their deadlines.
A schedule \( s \) is considered \emph{feasible} if \( C_i^s \leq \tilde{d_i} \) for every job \( i \in N \), or equivalently, if \( E_s \cup T_s = N \).

Following the formulation in \citep{Baptiste2010}, we adopt an equivalent \emph{maximization} approach rather than minimization: our goal is to maximize the weighted number of early jobs, with the constraint that each job must meet its deadline. Specifically, we aim to find a schedule \( s^\ast \) that maximizes $f(s) = \sum_{i \in E_s} w_i,$ subject to \( E_s \cup T_s = N \).

\section{Data-driven Approach for $1|\tilde{d}_i|\sum w_i U_i$ problem}\label{sec:Solution Approach}

~\cite{Pinedo2012} shows that if we know whether each job is early or tardy in an optimal schedule, we can solve a $1|\tilde{d}_i|\sum w_i U_i$ problem instance in polynomial time.
Indeed, we can construct an optimal schedule by arranging the jobs in non-descending order of $D_j$ ($j \in N$), where $D_j = d_j$ for early jobs and $D_j = \tilde{d}_j$ for tardy jobs.
Therefore, determining whether a given job is early or tardy is the main challenge.

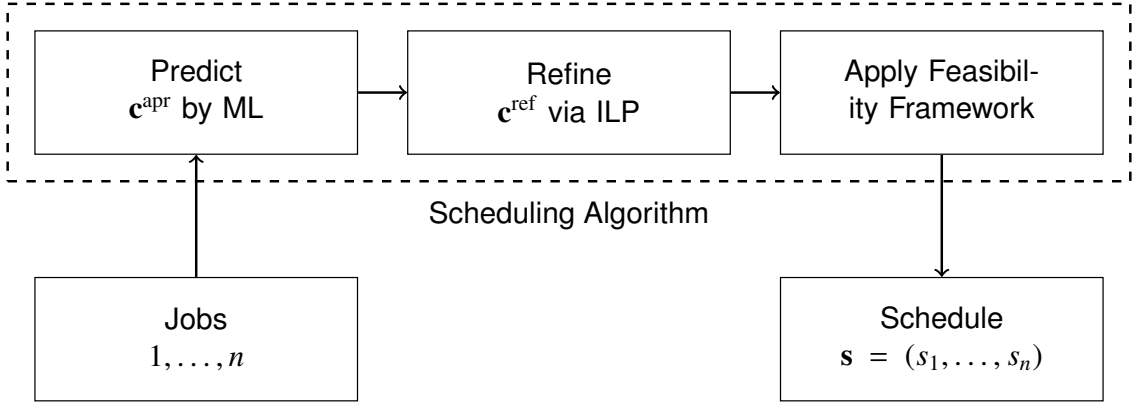
\begin{figure}[t]
\centering
\scalebox{0.9}{\resizebox{\textwidth}{!}{%
\begin{tikzpicture}[every node/.style={font=\sffamily, align=center}]
    \node (input) [draw, rectangle, minimum height=1.4cm, text width=3.4cm] at (0,-2.8) {
        Jobs \\ $1, \dots, n$
    };

    \node (classification) [draw, rectangle, minimum height=1.4cm, text width=3.4cm] at (0,0) {
        Predict \\ $\mathbf{c}^{\mathrm{apr}}$ by ML
    };

    \node (ilp) [draw, rectangle, minimum height=1.4cm, text width=3.4cm] at (4.2,0) {
        Refine \\ $\mathbf{c}^{\mathrm{ref}}$ via ILP
    };

    \node (feasibility) [draw, rectangle, minimum height=1.4cm, text width=3.4cm] at (8.4,0) {
        Apply Feasibility Framework
    };

    \node (output) [draw, rectangle, minimum height=1.4cm, text width=3.4cm] at (8.4,-2.8) {
        Schedule \\ $\mathbf{s} = (s_1, \dots, s_n)$
    };

    \draw[->, thick] (input.north) -- (classification.south);
    \draw[->, thick] (classification) -- (ilp);
    \draw[->, thick] (ilp) -- (feasibility);
    \draw[->, thick] (feasibility.south) -- (output.north);

    \node[draw, thick, dashed, inner sep=0.3cm, fit=(classification)(ilp)(feasibility),
      label={[yshift=-0.1cm]below:Scheduling Algorithm}] {};
\end{tikzpicture}%
}}
\caption{Overview of the proposed approach. We begin by using an ML-based oracle to predict jobs as early or tardy. Then, ILP is applied to refine some of these predictions. Finally, a feasibility framework generates a schedule based on the refined predictions.}
\label{fig:approach_overview}
\end{figure}

Figure~\ref{fig:approach_overview} presents an overview of our approach, where three interconnected components collaborate to achieve an optimal or near-optimal schedule. The ML model serves as the primary decision-making oracle in the initial step. In the next step, the least confident predictions are refined by solving to optimality a reduced, much smaller version of the original instance. Finally, a special algorithm is introduced to transform the sequence of predictions into a feasible solution.
We note that although we sometimes refer to the outputs of classifiers as probabilities, they are not meant to be calibrated. Instead, we interpret them as practical indicators of model confidence, i.e., \emph{prediction scores} used to guide decision-making in our scheduling heuristic.

\subsection{Predicting jobs labels}\label{subsec:classification_procedure}
We start by describing how the introduced ML-based oracle aids decision-making, while its actual implementation is discussed in Section~\ref{subsec:ml_models}.
Given a problem instance and its optimal solution $s^\ast$, the oracle $P_{apr}$ predicts whether a job $j \in N$ is \emph{early} or \emph{tardy} in $s^\ast$. 
The prediction outcomes are the class $c_j \in \{\textit{``early''}, \textit{``tardy''}\}$ assigned to job $j$ and the prediction score $Pr(j)$ indicating the confidence of the oracle in its decision.
This process is formalized in Algorithm~\ref{alg:Classify}.

\begin{algorithm}[ht]
\caption{\emph{Classify} function}\label{alg:Classify}
\begin{algorithmic}[1]
\Require set of jobs $N = \{1; 2; ... ; n\}$; \ oracle $P_{apr}$; \ threshold $\alpha \; (0 \leq \alpha \leq 1)$
\For {$j \in N$}
    \State $\textit{Pr}(j) \gets P_{apr}(j)$
    \State $c_j^{apr} \gets \textit{``early''} \;\; \textbf{if} \;\; \textit{Pr}(j) \geq \alpha \;\; \textbf{else} \;\; \textit{``tardy''}$
\EndFor
\State \Return $c_1^{apr}, ..., c_n^{apr}$; $\textit{Pr}(1), ..., \textit{Pr}(n)$
\end{algorithmic}
\end{algorithm}

\subsection{Refining predictions with ILP}\label{subsec:refine_predictions}
After executing Algorithm~\ref{alg:Classify} on a given problem instance, we obtain predicted classes \(c_j^{apr}\) and prediction scores \(\textit{Pr}(j)\) for each job \(j \in N\). However, scheduling the jobs immediately based solely on these predictions can be risky, as even one incorrect prediction may result in substantial deviations from the optimal solution. To mitigate this risk, we aim to refine the predictions.

The paper by~\cite{Baptiste2010} introduces the following theorem. 
Suppose we know whether a particular job $j \in N$ is classified as early ($D_j = d_j$) or tardy ($D_j = \tilde{d}_j$) in an optimal solution. We then formulate the \emph{reduced} problem on the set of jobs $N^\prime = N \setminus \{j\}$ with modified data as follows:

\begin{equation}
w_i^\prime = w_i, \;\; p_i^\prime = p_i \;\; (i \in N^\prime)
\end{equation}
\begin{equation}
d_i^\prime = 
\begin{cases} 
\min(d_i, D_j - p_j), & if \;\; d_i \leq D_j\\ 
d_i - p_j, & otherwise
\end{cases}
\;\;\;\;\; (i \in N^\prime) 
\end{equation}
\begin{equation}
\tilde{d}_i^\prime = 
\begin{cases} 
\min(\tilde{d}_i, D_j - p_j), & if \;\; \tilde{d}_i \leq D_j\\ 
\tilde{d}_i - p_j, & otherwise
\end{cases}
\;\;\;\;\; (i \in N^\prime)
\end{equation}

\begin{theorem}[Reduction theorem]
\label{the:reduction_theorem}
There exists a feasible schedule $s$ with an early set of jobs $E_s$ if and only if there exists a feasible schedule $s^\prime$ with an early set of jobs $E_s^\prime = E_s \setminus \{ j \} $ for the reduced problem.
\end{theorem}

We utilize this theorem to refine the predictions, reconsidering those of them which are the least confident according to the oracle. Initially, we select a set of jobs for which the oracle's predictions have the highest confidence, meaning the scores are close either to 0 or to 1. By applying the reduction theorem, these jobs are excluded from the instance, reducing the problem to only those jobs predicted with the smallest confidence (close to 0.5). The reduced instance can be solved using a general ILP solver, such as LINGO or Gurobi, and its optimal solution is then used to update the predictions for the original instance.

The ideas outlined above are formalized in Algorithm~\ref{alg:ILP}, referred to as the \texttt{Refine} function. 
We define the parameter \(\gamma \in \{0, \dots, n\}\) as the number of jobs to be handled by a general ILP solver.
Given the oracle outputs \((c_j^{\mathrm{apr}}, \Pr(j))\) from Algorithm~\ref{alg:Classify}, we sort the jobs in non-descending order of $|\Pr(j) - 0.5|$ and select the first \(\gamma\) jobs in the sorted list for refinement. The remaining jobs are assumed to be predicted reliably and are excluded from the instance using the reduction theorem, which results in a reduced problem instance. The ILP solver is then applied to the reduced instance (with a predefined time limit to guarantee quick termination). If a feasible solution \(s_1, \dots, s_\gamma\) is found, it replaces the corresponding predictions. Otherwise, the oracle's predictions remain unchanged.

\begin{algorithm}[ht]
\caption{\emph{Refine} function}\label{alg:ILP}
\begin{algorithmic}[1]
\Require set of jobs $N = \{1; 2; ... ; n\}$; \ $\gamma \in \mathbb{N} \ (0 \leq \gamma \leq n)$; \ time limit $\beta$ 
\Require predicted classes $c_1, ..., c_n$; predicted probabilities $Pr(j)$, ..., $Pr(j)$
\State $(j_1, ..., j_n) \gets Sort(N, |Pr(j) - 0.5|)$ \Comment{sort by $|Pr(j) - 0.5|$ non-desc.}
\State $N^\prime = N \setminus \{j_{\gamma+1}, ..., j_n\}$ \Comment{reduce the original instance}
\State $(s_1, ..., s_\gamma, \nu) \gets \textit{ILP}(N^\prime, \ \beta)$ \Comment{solve by ILP with time limit}
\State $(c_{j_1}^{ref}, ..., c_{j_\gamma}^{ref}) \gets (s_1, ..., s_\gamma) \;\; \textbf{if} \ \nu = \textit{``Optimal Solution Found''}$
\State \Return $c_1^{ref}, ..., c_n^{ref}$ \Comment{update predicted classes if a solution was found}
\end{algorithmic}
\end{algorithm}

\begin{algorithm}[ht]
\caption{Scheduling algorithm}\label{alg:SchedAlgo}
\begin{algorithmic}[1]
\Require set of jobs $N = \{1; 2; ... ; n\}$; \ predicted classes $c_1^{ref}, ..., c_n^{ref}$
\State \Return $\emptyset$ \;$\textbf{if} \;\; \text{EDF}(N) = \textit{``infeasible''}$
\State $D_i \gets d_i \;\; \textbf{if} \;\; c_i^{ref} = \textit{``early''} \;\; \textbf{else} \;\; \tilde{d}_i$ \;$(i\in N)$
\State $s \gets Sort(N, D_j)$ \Comment{jobs ordered by $D_i$ ascending}
\State $S \gets \emptyset$ \Comment{a set of scheduled jobs $S$}
\State $m \gets 1$ \Comment{a cursor $m$}
\While{$m \leq n$}\label{alg2while}
    \State $j \gets s(m)$ \Comment{consider the $m$-th job $j$ from $s$}
    \If{$c_j^{ref} = \textit{``early''}$} \Comment{if it is predicted as early}
        \State $\alpha = \textit{EDF}(N \setminus (S \cup \{j\}))$ \Comment{could we schedule the rest by \emph{EDF}}\label{EdfCheck}
        \State$\textit{schedNow} \gets true \;\; \textbf{if} \;\; \alpha = \textit{``feasible''} \;\; \textbf{else} \;\; false$
    \Else
        \State $\textit{schedNow}\gets\textit{true}$
    \EndIf
    \If{$\textit{schedNow}$} \Comment{schedule if \emph{early} \& passes \emph{EDF} or if \emph{tardy}}
        \State $S \gets S \cup \{j\}$ 
        \State $m \gets m + 1$
    \Else \Comment{otherwise, put $j$ further in $s$}
        \State $D_j \gets \tilde{d}_j$
        \State $s \gets Push(j, s)$ \Comment{new jobs order, $j$ is placed by $D_j = \tilde{d}_j$}
    \EndIf
\EndWhile
\State \Return $s$
\end{algorithmic}
\end{algorithm}

\subsection{Scheduling Algorithm.}\label{subsec:feas_algo}
Assume that we have executed Algorithm~\ref{alg:Classify} followed by Algorithm~\ref{alg:ILP}, obtaining the predicted classes \(c_1^{ref}, \dots, c_n^{ref}\). However, this sequence of predictions does not \emph{guarantee} a feasible schedule. To ensure feasibility and to convert the predictions into a valid schedule, an additional step is required.
Further on, we use the fact that a given problem is feasible \emph{if and only if} scheduling jobs in non-descending order of their deadlines yields a feasible solution~\citep{Pinedo2012}.
We refer to this check as the \emph{EDF check} (\emph{Earliest Deadline First check}).

Algorithm~\ref{alg:SchedAlgo} outlines the scheduling procedure. We begin by checking whether the instance is feasible. If so, we sort the jobs by their values \(D_i\), where \(D_i = d_i\) if job \(i\) is predicted as early, and \(D_i = \tilde{d}_i\) otherwise. This defines an initial permutation of jobs, denoted as \(s\).
We initialize a cursor \(m\) at the start of \(s\) and begin with the first job $j$ in $s$. If the current job \(j\) is predicted as tardy, we schedule it immediately and move to the next one (lines 11--12, 14--16). If \(j\) is predicted as early, we check whether the remaining unscheduled jobs can still be scheduled using an EDF check (line 9). If the check is passed, we schedule \(j\) and continue. Otherwise, we change its predicted class to tardy, update \(D_j\) to \(\tilde{d}_j\), and reinsert \(j\) into \(s\) so that the list remains sorted (lines 17--19).
The entire procedure continues until all jobs are scheduled. The final schedule corresponds to the updated order in \(s\), with the cursor positioned after the last job.

We remark that the EDF check can be efficiently implemented. First, we sort all jobs only once at the beginning of Algorithm~\ref{alg:SchedAlgo} according to their deadlines $\tilde{d}_i$. Then during each EDF check we go through this pre-sorted list, considering only the jobs that remain unscheduled. For each job, we verify whether adding its duration to the current total length of the schedule would still meet its deadline. If so, we add the job’s duration to the total length and proceed to the next job; otherwise, the check fails. As for the practical impact, we note that while EDF checks dominate the runtime of the rule-based heuristic, the overall time spent on them remains negligible (see Section~\ref{sec:Experimental Results} for details).

\begin{proposition}
Algorithm~\ref{alg:SchedAlgo} finds a feasible schedule if one exists.
\end{proposition}

\begin{proof}
We must demonstrate that scheduling a job \( j \) allows to schedule all remaining jobs without violating their deadlines. We proceed by considering two mutually exclusive cases based on the predicted class of job \( j \) and the outcome of the EDF check.

\medskip

\noindent 
Case 1: Job \( j \) has an early predicted class. Then, if it passes the EDF check, feasibility is trivially preserved. If the EDF check fails, scheduling of $j$ is postponed, leaving the feasibility of the remaining jobs unchanged.

\medskip

\noindent 
Case 2: Job \( j \) has a tardy predicted class. 
Two observations can be made in advance.  
First, the jobs in \( s \) are always kept sorted, so the sorting key \( D_j \) of job \( j \) is always the smallest value of \( D \) among the remaining unscheduled jobs.  
Second, since \( j \) is predicted as tardy, \( D_j = \tilde{d}_j \).  
Therefore, scheduling \( j \) is identical to the very first step of scheduling all the remaining jobs by the \emph{EDF} rule, and the remaining unscheduled jobs can be scheduled by running the \emph{EDF} until the end.  
Hence, the ability to construct a feasible schedule is preserved.  
This completes the proof.
\end{proof}

\begin{remark}
The algorithm terminates in at most \(2n\) steps, as each job is handled at most twice: once during initial evaluation and once when revisited.
\end{remark}

\section{Machine Learning Methodology}\label{sec:ML Methodology}

In the previous section, we introduced the concept of a decision oracle in a general manner. 
Here, we delve into the practical aspects of its implementation.
When applying machine learning to a problem, two main challenges arise. 
The first is to identify relevant problem features, so an ML model can generalize well across different distributions. Considering many types of ML models, the second challenge is to select the one that balances prediction accuracy, training time, and inference speed for the given problem.
We discuss each of these topics below.
Further on, $\textbf{x}$ denotes the vector $(x_1, \dots, x_n)$, and abbreviations $avg$ and $std$ stand for ``average'' and ``standard deviation'' respectively.
We also assume the natural logarithm $\ln\textbf{x}$ is applied component-wise.

\subsection{Developing robust features}\label{subsec:developing_robust_features}
We associate each job $j$ with an eight-dimensional vector of parameters, which includes weight $w_j$, duration $p_j$, due date $d_j$, deadline $\tilde{d_j}$, and four derived parameters: $\frac{w_j}{p_j}$, $w_j - p_j$, $\frac{d_j}{\tilde{d_j}}$, $\tilde{d_j} - d_j$.
While these parameters could be directly used as features in a machine learning model, we found a more effective featurization approach.
Let $x_j$ denote any of the eight mentioned above parameters of a job $j$, i.e., $x_j \in \{w_j, p_j, d_j, \tilde{d_j}, \frac{w_j}{p_j}, \frac{d_j}{\tilde{d_j}}, w_j - p_j, \tilde{d_j} - d_j\}$.
We propose two types of features.
The idea behind the first type is to consider the deviation of the parameter from the average value:
\begin{equation}\label{eq:abs_difference}
x_j^{\textit{dev}} = \frac{x_j - \textit{avg}(\textbf{x})}{\textit{std}(\textbf{x})}.
\end{equation}

In the machine learning community, this technique is referred to as calculating the z-score. Notably, using it is rather a way to express how a particular \emph{job parameter} relates to an \emph{aggregated value} obtained across all jobs. 
Preliminary experiments have shown that considering features in the form of Equation~(\ref{eq:abs_difference}) significantly improves prediction accuracy compared to the raw parameter values $x_j$.

The idea of another type of features we use is connected to a relative difference in the logarithmic scale:
\begin{equation}\label{eq:rel_difference}
    x_j^{\textit{rel}} = \frac{\ln(x_j) - \textit{avg}(\ln\textbf{x})}{\textit{std}(\ln\textbf{x})}.
\end{equation}

The meaning of this definition can become clearer if we consider the following expression:
\begin{align}
    x_j^{\textit{rel}}
    = \frac{\ln(x_j) - \textit{avg}(\ln\textbf{x})}{\textit{std}(\ln\textbf{x})}
    = \frac{\ln\left(\frac{x_j}{e^{\textit{avg}(\ln\textbf{x})}}\right)}{\ln(e^{\textit{std}(\ln\textbf{x})})}
    = \frac{1}{\textit{std}(\ln\textbf{x})} \ln\left(\frac{x_j}{e^{\textit{avg}(\ln\textbf{x})}}\right),
\end{align}
which implies:
\begin{align}\label{eq:rel_diff_prop}
    x_j^{\textit{rel}}
    \sim \ln\left(\frac{x_j}{e^{\textit{avg}(\ln\textbf{x})}}\right)
    = \ln(x_j) - \textit{avg}(\ln\textbf{x}).
\end{align}
Similarly to \( x_j^{\textit{dev}} \), we aim to consider \( x_j \) in relation to an aggregate function over \( \mathbf{x} \), that is, a function defined on the entire sample.
Modeling such a relation as a fraction may result in extremely high or low values of the features, which can negatively affect the model accuracy during training.
Therefore, we consider the natural logarithm, which transforms the ratio into a difference~\eqref{eq:rel_diff_prop} and ensures better numerical stability. 

Consequently, \emph{each job is represented by sixteen features}: eight of them follow the form of Equation \eqref{eq:abs_difference} and another eight align with Equation \eqref{eq:rel_difference}.
Importantly, these features reflect the combinatorial nature of the scheduling problem, as each job parameter is considered in relation to the value from all the jobs in the instance. 
Further discussion on this topic is provided in the next subsection.

\subsection{ML model development}\label{subsec:ml_models}
In Section~\ref{subsec:classification_procedure}, we use a decision-making oracle to predict jobs as early or tardy.
Here, we propose several ML models for implementing the oracle.

As it is impractical to examine all known ML models, we need a systematic approach to select a representative subset of models for further comparison.
Our idea comes from the nature of the addressed scheduling problem: we consider how much a job's classification depends on information about other jobs in the instance.
We define a model as \emph{locally informed} if its predictions are based solely on features associated with the job under consideration. These features may include aggregate statistics derived from the instance but do not rely on direct access to other jobs' features during inference. In contrast, we define a model as \emph{globally informed} if it explicitly utilizes features of multiple or all jobs in the instance (e.g., through attention mechanisms or concatenation of feature vectors).
Thus, our guiding principle for creating a sample of diverse ML models is how much information about the other jobs they use to make a decision.

Below, we examine two globally informed classification methods. The first relies on conditional probability estimates, which only partially incorporate information from other jobs. 
The second uses attention mechanisms, considering information from all jobs in the instance.
We then turn to locally informed classification, focusing on AutoML tools and a particularly effective perceptron model.
All the approaches are thoroughly compared in Section~\ref{subsec:training}.

\subsubsection{Globally informed decisions based on conditional probabilities}\label{subsubsec:cond_probs}

Certain theoretical statements about the $1|\tilde{d}_i|\sum w_i U_i$ problem can help to develop a context for decision-making.
Consider the following theorem, presented by~\cite{Baptiste2010}:

\begin{theorem}[Dominance rule]\label{the:dominance_theorem}
Consider two jobs, $i$ and $j$, satisfying the conditions $w_j > w_i$, $p_j \leq p_i$, $d_j \geq d_i$, and $\Tilde{d}_j \leq \Tilde{d}_i$.
Let $s^\ast$ denote an optimal schedule.
Then, if job $i$ is scheduled early in $s^\ast$, the same holds for job $j$; similarly, if job $j$ is tardy in $s^\ast$, then job $i$ is also tardy.
\end{theorem}

An important conclusion of this theorem is that deciding on one job can impact the scheduling of another. 
In practice, many pairs of jobs follow the dominance rule pattern even if a few inequalities are violated.
Thus, our idea is to express the dominance rule in terms of probability.
Consider two mutually exclusive events that a job $j$ is early ($j_\mathcal{E}$) or tardy ($j_\mathcal{T}$) in a fixed optimal solution.
Given another job $i$, we refer to the marginal probability:
\begin{equation}\label{eq:marginal_prob}
\textit{Pr}(j_\mathcal{E}, i) = \textit{Pr}(j_\mathcal{E} \mid i_\mathcal{E}) \ \textit{Pr}(i_\mathcal{E}) + \textit{Pr}(j_\mathcal{E} \mid i_\mathcal{T}) \ \textit{Pr}(i_\mathcal{T}).
\end{equation}

Two distinct perceptrons can be used to predict the values on the right-hand side.
The first computes the apriori estimates $\textit{Pr}(i_\mathcal{E})$ and $\textit{Pr}(i_\mathcal{T})$, while the second handles conditioned terms $\textit{Pr}(j_\mathcal{E} \mid i_\mathcal{E})$ and $\textit{Pr}(j_\mathcal{E} \mid i_\mathcal{T})$.
A decision regarding job $j$ is based on a rounded average of various marginal estimates $\textit{Pr}(j_\mathcal{E}, i)$, calculated with respect to different jobs $i$.

Such an approach to decision-making has several advantages. 
At first, instead of only focusing on job $j$, we also consider other jobs in the instance, which can improve prediction accuracy. 
Secondly, our approach balances apriori and conditional probability estimates, as they come from independent oracles. 
Lastly, the proposed way of decision-making can be combined with Theorem~\ref{the:dominance_theorem}, which eliminates the need to estimate conditional probabilities in certain cases.

It might seem useful to decide on $j$ based on several jobs.
However, extending the proposed approach to incorporate multiple jobs faces significant challenges.
Indeed, a decision based on two jobs would require already four terms on the right-hand side of the Equation~\eqref{eq:marginal_prob}, while generally, the number of terms grows exponentially.
In addition, one must also train an exponentially rising number of oracles.

\subsubsection{Globally informed decisions provided by attention}
Maximizing prediction accuracy can be taken to the extreme if a decision about some job is based on the context of all other jobs in the instance.
This concept is effectively addressed by a neural network with \emph{attention layer}, which can learn an aggregation function over a set or sequence of items \citep{Attention2017}.

The attention layer transforms a sequence of feature vectors $h(1), \dots, h(n)$ into another sequence $h^\prime(1), \dots, h^\prime(n)$. 
The key property of this transformation is that each vector $h^\prime(j)$ can potentially contain information from all feature vectors $h(1), \dots, h(n)$. 
The term \emph{attention} highlights the ability to focus on relevant elements in the sequence while ignoring others. 
To compute $h^\prime(j)$, we apply learned linear transformations to all vectors $h(1), \dots, h(n)$ to obtain queries, keys, and values (see details in~\cite{Attention2017}). 
Then, $h^\prime(j)$ is computed as a weighted sum of the value vectors, where the weights depend on the similarity (e.g., dot product) between the query for job~$j$ and the keys of all jobs.
These weights are learned from data.
We remark that an attention layer is inherently permutation-invariant, as it aggregates information based on learned weights rather than on fixed order of input. Since jobs in $N$ have no predefined order, the attention model naturally processes jobs independently of their order in the input sequence.

In our scenario, the sequence passed to the attention layer corresponds to the sequence of feature vectors $h(j)$, $j \in N$. 
Consequently, we anticipate that the attention layer can learn to produce a new feature vector $h^\prime(j)$ by extracting the relevant information from other jobs in the instance. 
As it is typically done in ML, we use the attention layer as part of a larger model, which interleaves two attention layers with 2-layer MLPs with 80 neurons per layer and ReLU activation function.
Once the output from the final attention layer is obtained for a specific job, a linear classification layer predicts the job as early or tardy, which is done simultaneously for all the jobs in the instance.
The classification loss is then used to update the entire model. 

The main drawback of this model is the quadratic complexity of computing the weighting scores with respect to the sequence length. Therefore, the model could be significantly slower than the other approaches. 
On the other hand, we get all predictions in a single forward pass, unlike the approach based on conditional probabilities described earlier, where we have to repeat the inference process for every job.

\subsubsection{Locally informed classification with AutoML}
Selecting the best ML model can be challenging due to the abundance of available algorithms and the need to fine-tune their hyperparameters.
Automated Machine Learning (AutoML) is a framework designed to address exactly this challenge. 
It comprises various techniques to automate the entire model development process: data preprocessing, features engineering, model selection, and hyperparameter tuning.
Using AutoML not only reduces the effort in model development but also provides high-quality baseline models.  
In this study, we consider three AutoML frameworks: TabPFN, AutoGluon, and AutoSklearn.
The development of the ML model is analogous across all three tools.
Given a training dataset, the AutoML framework identifies the best-fitting ML model.

\emph{TabPFN} is a transformer model trained on synthetic data to emulate real-world tabular datasets~\citep{Hollman2023}.
It is designed for supervised classification, achieving state-of-the-art performance: the authors report that an analysis of 18 small numerical datasets shows the superiority of TabPFN over individual base-level classification algorithms and its competitive performance with leading AutoML frameworks in significantly less time.
As our training data can be expressed in a tabular form, using \emph{TabPFN} is highly relevant, especially since it can provide knowledge about the baseline accuracy we can achieve.
The only drawback of this model is that it can currently be trained only on small tabular datasets (1000 training examples, 100 numerical features, and 10 classes).

\emph{AutoGluon} is an AutoML framework designed specifically for tabular data \citep{Erickson2020}.
It includes simple algorithms for data preprocessing, four types of feature engineering approaches, and a wide range of models for tabular predictions, such as KNN, neural networks, LightGBM trees, random forests, and XGBoost. 
It automates exploring model architectures and hyperparameter spaces by incorporating Bayesian optimization and neural architecture search. 
By using AutoGluon in the context of our problem, we explore more than a dozen different ML models and benefit from Bayesian optimization for hyperparameter tuning, attaining results comparable to their manual exploration.

\emph{AutoSklearn} is another advanced AutoML framework tailored for automating model selection, hyperparameter optimization, and feature preprocessing~\citep{Feurer2020}.
It extends the widely used scikit-learn library~\citep{ScikitLearn2011} by integrating Bayesian optimization, automated ensemble construction, and meta-learning techniques to efficiently navigate the space of machine learning pipelines.
In contrast to AutoGluon, AutoSklearn places a strong emphasis on model ensembling and warm-starting hyperparameter optimization using prior knowledge from related tasks.
In our study, employing AutoSklearn enables an independent exploration of model configurations, providing additional diversity in candidate models and serving as a valuable cross-check against solutions found by other frameworks.

\subsubsection{Multilayer perceptron architecture}
We highlight a particularly useful model configuration: a multilayer perceptron (MLP) with specific hyperparameters. 
The model begins with an input layer of 16 neurons, reflecting the dimension of feature vectors. 
It is followed by two hidden layers, each containing 80 neurons -- a choice based on a preliminary experiment comparing several common sizes (16, 40, 80, 120).
We found that 80 neurons offer a good trade-off between predictive accuracy and inference time across datasets, though the differences across tested configurations were within a few fractions of a percent.
The final output layer contains two neurons representing scores for early and tardy job categories. 
While having a single output neuron is also possible, we found that using two outputs is more convenient when applying cross-entropy as the loss function.
We use Rectified Linear Unit (ReLU) as a standard activation function, widely adopted in neural network models.

\section{Experimental Results}\label{sec:Experimental Results}

In this section, we first provide an overview of the datasets utilized in our experiments. 
Subsequently, we describe the conducted experiments, focusing on model selection and comparison with state-of-the-art methods.
The proposed algorithm is implemented in Python using the PyTorch library for neural network development. 
The training labels are obtained on a cluster node (18 Cores/CPU; 2.3GHz; 256 GB RAM), while the tests are performed in Google Colab. 
The code and data are available at \href{https://doi.org/10.5281/zenodo.17233362}{https://doi.org/10.5281/zenodo.17233362}
.

\subsection{Datasets description}\label{subsec:datasets}

\begin{table}[!htbp]
    \scriptsize
    \centering
    \begin{tabular}{| p{.06\textwidth} | p{.27\textwidth} | p{.56\textwidth} | }
    \hline
    \textbf{ID} 
    & \textbf{Generation Procedure}
    & \textbf{Description} \\
    \hline
    1
    &
    $w_i \sim \mathcal{N}(50, 20)$
    \newline
    $p_i \sim \mathcal{U}(1, 100)$
    \newline
    $d_i \sim \mathcal{U}(0.3\cdot\sum{p_i}, 0.7\cdot\sum{p_i})$
    \newline
    $\tilde{d_i} \sim \mathcal{U}(d_i, 1.1\cdot\sum{p_i})$
    &
    The dataset illustrates a scheduling scenario where processing each job yields almost equal benefits, yet there is considerable variation in how difficult each task is to complete.
    \\
    \hline
    2
    &
    $w_i \sim \mathcal{U}(30, 80)$
    \newline
    $p_i \sim \mathcal{N}(50, 10)$
    \newline
    $d_i \sim \mathcal{N}(0.5\cdot\sum{p_i}, 0.1\cdot\sum{p_i})$
    \newline
    $\tilde{d_i} \sim \mathcal{U}(d_i, 1.2\cdot\sum{p_i})$
    &
    The dataset represents a scheduling scenario, where jobs have diverse profits, while their durations are close to each other with a slight variation \citep{Novak2022}.
    \\
    \hline
    3
    &
    $w_i = 2 \cdot p_i + 20$
    \newline
    $p_i \sim \mathcal{N}(40, 15)$
    \newline
    $d_i \sim \mathcal{U}(0.3 \sum{p_i}, 0.7\cdot\sum{p_i})$
    \newline
    $\tilde{d_i} \sim \mathcal{U}(d_i, 1.1\cdot\sum{p_i})$
    &
    The dataset shows a direct correlation between job weight and processing time, reflecting situations where more beneficial tasks require proportionally longer time \citep{Hejl2022}.
    \\
    \hline
    4
    &
    $w_i = p_i^2 + 10$
    \newline
    $p_i \sim \mathcal{N}(35, 10)$
    \newline
    $d_i \sim \mathcal{U}(0.3\cdot\sum{p_i}, 0.7\cdot\sum{p_i})$
    \newline
    $\tilde{d_i} \sim \mathcal{U}(d_i, 1.1\cdot\sum{p_i})$
    &
    The dataset introduces a quadratic correlation between job weight and duration. It is relevant for industries where the weight (profit or cost) of a job increases rapidly alongside its processing time.
    \\
    \hline
    5
    &
    $w_i \sim \mathcal{U}(20, 80)$
    \newline
    $p_i \sim \mathcal{N}(45, 15)$
    \newline
    $d_i \sim \mathcal{U}(0.5\cdot\sum{p_i}, 0.8\cdot\sum{p_i})$
    \newline
    $\tilde{d_i} = d_i + \frac{n}{5} \cdot w_i$
    &
    In this dataset the weight of a job is uniformly distributed, and there is a linear correlation between the weight and deadline. It represents scenarios where valuable jobs require more time to be completed.
    \\
    \hline
    6
    &
    $w_i = 100 / (p_i + 1)$
    \newline
    $p_i \sim \mathcal{N}(40, 10)$
    \newline
    $d_i \sim \mathcal{U}(0.3\cdot\sum{p_i}, 0.7\cdot\sum{p_i})$
    \newline
    $\tilde{d_i} \sim \mathcal{U}(d_i, 1.1\cdot\sum{p_i})$
    &
    The dataset introduces a non-linear inverse correlation between job weight and processing time. It is relevant for scenarios where easier tasks have a higher priority.
    \\
    \hline
    7
    &
    $w_i \sim \mathcal{U}(10, 60)$
    \newline
    $p_i \sim \emph{Exp}(30)$
    \newline
    $d_i \sim \mathcal{U}(0.3\cdot\sum{p_i}, 0.7\cdot\sum{p_i})$
    \newline
    $\tilde{d_i} \sim \mathcal{U}(d_i, 1.1\cdot\sum{p_i})$
    &
    The dataset represents a scenario where job weights vary uniformly and durations follow an exponential distribution.
    It is relevant for tasks with inherent variability in processing times \citep{Kaandorp2007}.
    \\
    \hline
    8
    &
    $w_i \sim \mathcal{LN}(3, 1)$
    \newline
    $p_i \sim \mathcal{LN}(4, 1)$
    \newline
    $d_i \sim \mathcal{U}(0.3\cdot\sum{p_i}, 0.7\cdot\sum{p_i})$
    \newline
    $\tilde{d_i} \sim \mathcal{U}(d_i, 1.1\cdot\sum{p_i})$
    &
    The dataset uses log-normal distributions for weight and processing time, which can be relevant for modeling real-world scenarios where certain jobs have highly skewed distributions \citep{Lee2024}.
    \\
    \hline
    9
    &
    $w_i = 1.5 \cdot p_i + 0.2 \cdot d_i$
    \newline
    $p_i \sim \mathcal{N}(40, 10)$
    \newline
    $d_i \sim \mathcal{U}(0.3\cdot\sum{p_i}, 0.7\cdot\sum{p_i})$
    \newline
    $\tilde{d_i} \sim \mathcal{U}(d_i, 1.1\cdot\sum{p_i})$
    &
    The dataset combines linear correlations between weight and processing time and weight and deadline. It can be relevant for scenarios where both longer duration and due date influence the profit.
    \\
    \hline
    10
    &
    $w_i \sim \mathcal{LN}(4, 2)$
    \newline
    $p_i \sim \emph{Exp}(40)$
    \newline
    $d_i \sim \mathcal{N}(0.5\cdot\sum{p_i}, 100)$
    \newline
    $\tilde{d_i} \sim \mathcal{N}(2 \cdot d_i, 200)$
    &
    The dataset presents an edge case with a skewed distribution of weights and significant variability in both processing times and due dates. It is useful to explore the robustness of a scheduling approach.
    \\
    \hline
    11-15
    &
    $w_i \sim \mathcal{U}(1, 100)$
    \newline
    $p_i \sim \mathcal{U}(1, 100)$
    \newline
    $d_i \sim \mathcal{U}(a\cdot\sum{p_i}, b\cdot\sum{p_i})$
    \newline
    $\tilde{d_i} \sim \mathcal{U}(d_i, 1.1\cdot\sum{p_i})$
    \newline
    $(a, b) \in \{(0.1, 0.3), (0.1, 0.7),$
    \newline 
    $(0.3, 0.5), (0.3, 0.7), (0.5, 0.7)\}$
    &
    These five datasets, each defined by a different $(a, b)$ pair, are intended to make a fair comparison with the approaches in \citep{Baptiste2010}. They offer variations in uniform distributions for job weights, processing times, and due dates, allowing to assess the impact of parameterized due dates on a scheduling algorithm.
    \\
    \hline
\end{tabular}
    \caption{Datasets used in the experiments}
    \label{tab:datasets}
\end{table}

Traditional evaluation methods in operational research, like those demonstrated in \citep{Baptiste2010, Hejl2022}, test the performance of a scheduling algorithm based on \emph{uniform distribution} $\mathcal{U}(a, b)$ of job parameters. 
However, our study broadens this scope by considering a variety of distributions reflecting real-world scenarios.
For instance, \cite{Novak2022} justifies the consideration of \emph{normal distribution} $\mathcal{N}(\mu, \sigma)$, reflecting scenarios such as production stages with human workers facing uncertain assembly times. 
Similarly, \cite{Xu2020} utilize normal distribution to model uncertainties in renewable energy generation, particularly wind power.
Other papers, such as \citep{Wang1999} and \citep{Kaandorp2007}, focus on \emph{exponentially distributed} $Exp(\lambda)$ service times. 
\cite{Lee2024} explore sequencing rules for various service-time distributions, including non-identical exponential and log-normal distributions, relevant in healthcare modeling. \cite{Ren2023} notes the prevalence of \emph{log-normal distribution} $\mathcal{LN}(\mu, \sigma)$ in the one-day travel mileage of electric private cars.

Motivated by this prior research utilizing various distributions, we create diverse datasets to comprehensively evaluate our algorithm under realistic conditions. 
Table~\ref{tab:datasets} defines the generation procedure for every dataset and provides an intuition on its relevance to real-world situations.
For each dataset, we generate multiple instances comprising from 500 to 5000 jobs and solve them using the exact algorithm proposed by~\cite{Baptiste2010}.
Job labels (early/tardy) are determined based on their status in the optimal solution, serving as the ground truth for training ML models (see Section~\ref{subsec:training}).  
The optimal values are used as benchmarks for evaluating our approach compared to the state-of-the-art (see Section~\ref{subsec:results}).

\subsection{Experiments with ML models}\label{subsec:training}
In these experimental series, we compare multiple machine learning architectures discussed in Section~\ref{subsec:ml_models}. 
The results are detailed in Table~\ref{tab:ml_models}.

\begin{table*}[t]
    \centering
    \footnotesize
    \begin{tabular}{lcccccc}
    \toprule
    \multirow{2}{*}{\textbf{Dataset}} & \multicolumn{4}{c}{\textbf{Locally informed}} & \multicolumn{2}{c}{\textbf{Globally informed}} \\
    \cmidrule(lr){2-5} \cmidrule(lr){6-7}
    & MLP & AutoGluon & AutoSklearn & TabPFN & CondProbs & Attn \\
    \midrule
    \multicolumn{7}{c}{\textbf{Validation Accuracy, [\%]}} \\
    \midrule
    1  & 98.3  & 98.5  & 98.5  & 98.1  & 98.3  & \textbf{98.7} \\
    2  & 97.1  & \textbf{97.7}  & \textbf{97.7}  & 97.2  & 97.1  & \textbf{97.7} \\
    3  & 97.6  & \textbf{98.0}  & 97.8  & 96.7  & 97.6  & 97.7 \\
    4  & 97.8  & \textbf{98.3}  & \textbf{98.3}  & 97.2  & 97.8  & \textbf{98.3} \\
    5  & 97.2  & 97.4  & 97.4  & 95.8  & 97.2  & \textbf{97.8} \\
    6  & 97.9  & \textbf{98.3}  & 98.2  & \textbf{98.3}  & 97.9  & 98.2 \\
    7  & 98.5  & 98.6  & 98.6  & \textbf{99.0}  & 98.5  & 98.7 \\
    8  & 97.7  & 97.8  & 97.9  & 97.9  & 97.7  & \textbf{98.5} \\
    9  & 98.2  & 98.5  & 98.4  & 97.3  & 98.2  & \textbf{98.8} \\
    10 & 98.3  & 98.5  & 98.5  & 97.6  & 98.3  & \textbf{99.0} \\
    11 & 97.4  & \textbf{97.9}  & 97.8  & 96.8  & 97.4  & 97.7 \\
    12 & 97.2  & 97.7  & 97.7  & 96.0  & 97.2  & \textbf{97.8} \\
    13 & 98.1  & 98.4  & 98.3  & 98.0  & 98.1  & \textbf{98.5} \\
    14 & 98.3  & 98.3  & 98.4  & 97.4  & 98.3  & \textbf{98.7} \\
    15 & 98.5  & 98.6  & 98.6  & 97.7  & 98.5  & \textbf{98.8} \\
    \midrule
    \multicolumn{7}{c}{\textbf{Average Training Time, [s]}} \\
    \midrule
    1-15  & 720 & 4740 & 3750 & \textbf{120} & 1560 & 2040 \\
    \midrule
    \multicolumn{7}{c}{\textbf{Per-Job Inference Time, [ms]}} \\
    \midrule
    1-15  & \textbf{0.003} & 0.481 & 0.230 & 2.614 & 0.008 & 20.68 \\
    \bottomrule
\end{tabular}
    \caption{Validation accuracy of ML models, average training and inference time}
    \label{tab:ml_models}
\end{table*}

For each dataset in Table~\ref{tab:datasets}, we generate 2000 instances with 500 jobs and 2000 with 1000 jobs, yielding 3 million job-level samples. Each dataset is split 80/20 into training and validation sets. Using the features from Section~\ref{subsec:developing_robust_features}, we train several machine learning models described in Section~\ref{subsec:ml_models}: a Multilayer Perceptron (MLP), three AutoML frameworks (AutoGluon, AutoSklearn, TabPFN), and two context-based models (CondProbs and Attn) based on conditional probabilities and attention mechanisms. Due to memory limits, AutoGluon and AutoSklearn are trained on 750,000 samples (500 instances each of 500 and 1000 jobs), while TabPFN is limited to 1000 samples. We evaluate all models by validation accuracy (in percentages) and inference time per job (in milliseconds).

As shown in Table~\ref{tab:ml_models}, the attention-based model achieves the highest validation accuracy on eleven out of fifteen datasets. However, it suffers from extremely high inference times: for a simple instance with 1000 jobs, predicting all labels takes over 20 seconds.
Meanwhile, the model based on conditional probabilities -- combined with the features we designed -- performs comparably to our proposed MLP on most datasets. This is expected, as the first network modeling the a priori estimates shares the same architecture as our MLP, and the second network replacing the conditional estimate offers no significant improvement over the baseline.
Thus, we do not further consider the conditional probabilities model and instead focus on the MLP.

Due to the slow inference of the attention model and the redundancy of the conditional probabilities model, we focus on the MLP and AutoML models.
While AutoML frameworks achieve about 0.5\% higher validation accuracy, the MLP is at least 75 times faster than the fastest AutoML option (AutoSklearn).
This makes the MLP the best trade-off between accuracy and inference speed.
Choosing the MLP brings additional benefits.
The MLP provides a fixed, interpretable architecture (defined by a consistent number of layers, neurons per layer, and activation functions), whereas AutoML frameworks yield different architectures across datasets.
Moreover, AutoGluon and AutoSklearn heavily rely on decision tree ensembles, which are known to be prone to overfitting \citep{Amro2021}.
For these reasons, we proceed with the MLP architecture in the remainder of this work.

\begin{figure}[t]
    \centering
    \begin{tikzpicture}
\begin{groupplot}[
    group style={
        group size=2 by 1,
        horizontal sep=2cm,
    },
    width=0.45\textwidth,
    height=0.35\textwidth,
    xlabel={Predicted probability},
    enlargelimits=false,
    ymin=0,
    grid=none,
    axis line style={very thick},
    tick style={very thick},
    label style={font=\small},
    tick label style={font=\small},
]

\nextgroupplot[ylabel={Frequency of errors}]
\addplot+[
    ybar,
    bar width=1pt,
    fill=cyan!50,
    draw=none,
] table [x index=0, y index=1] {
    0.05 0.001
    0.10 0.005
    0.15 0.012
    0.20 0.042
    0.25 0.096
    0.30 0.179
    0.35 0.284
    0.40 0.372
    0.45 0.456
    0.50 0.50
    0.55 0.344
    0.60 0.202
    0.65 0.148
    0.70 0.080
    0.75 0.054
    0.80 0.027
    0.85 0.012
    0.90 0.005
    0.95 0.002
};
\addplot[
    red,
    thick,
    domain=0:1,
    samples=200,
] {0.52 * exp(-25*(x-0.5)^2)};
\addplot[
    black,
    thick,
    dashed,
] coordinates {(0.5,0) (0.5,0.55)};

\nextgroupplot[
    ylabel={Frequency of occurrence},
    ymin=0,
    ymax=0.5,
]
\addplot[
    blue,
    thick,
    domain=0.002:0.997,
    samples=300,
] {0.001/(0.00005 + x*(1-x))};
\end{groupplot}
\end{tikzpicture}
    \caption{Error frequency (left) and distribution of predicted probabilities (right).}
    \label{fig:neuralnet}
\end{figure}
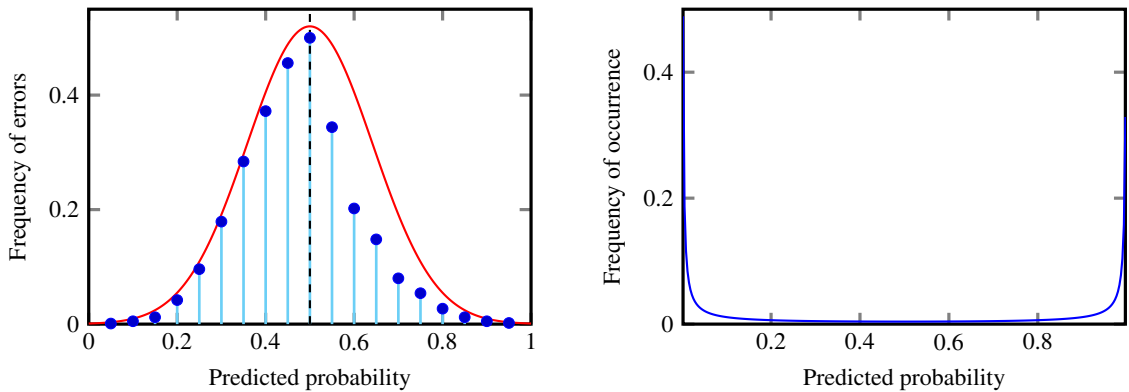

\paragraph{Reliability of Model Confidence Scores}
We assessed how well the MLP’s prediction scores correspond to actual prediction accuracy using 500{,}000 training samples.
The left plot in Figure~\ref{fig:neuralnet} shows the empirical error rate as a function of predicted probability, aggregated in bins of width 0.05. As expected, errors are most frequent around prediction scores of 0.5 and rare near 0 or 1, resulting in an approximately bell-shaped curve centered at 0.5.
The right plot shows the smoothed distribution of predicted probabilities, with most predictions concentrated near 0 or 1. This indicates that the model typically makes high-confidence predictions.
Taken together, these results suggest that when the model is confident -- which is the case for most predictions -- it is also very likely to be correct. This is a desirable property when using its outputs for heuristic decision-making.

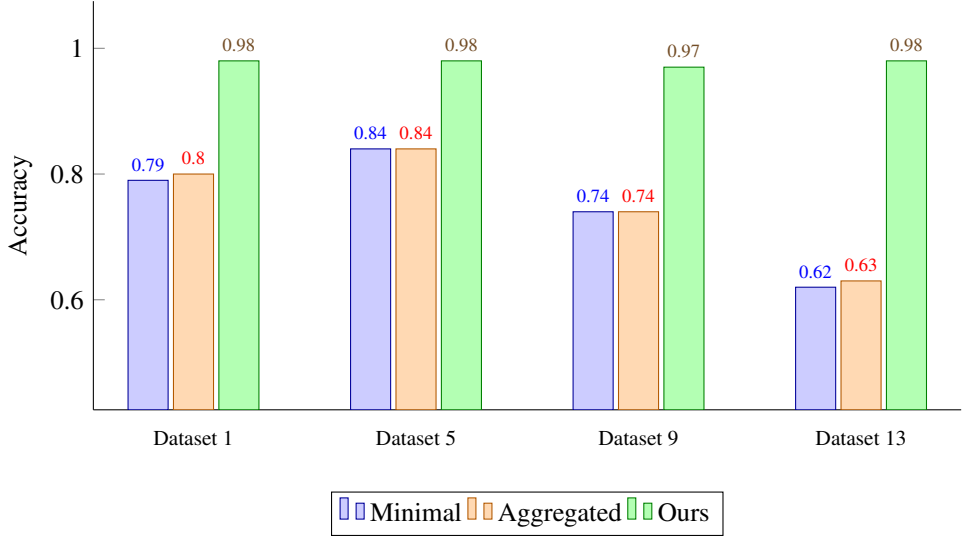
\begin{figure}[!htbp]
    \centering
    \begin{tikzpicture}
\begin{axis}[
    ybar,
    bar width=15pt,
    width=13cm,
    height=7cm,
    enlargelimits=0.15,
    ylabel={Accuracy},
    ymin=0.5, ymax=1.0,
    ylabel near ticks,
    symbolic x coords={Dataset 1, Dataset 5, Dataset 9, Dataset 13},
    xtick=data,
    xtick style={draw=none},
    x tick label style={font=\footnotesize},
    legend style={at={(0.5,-0.2)}, anchor=north, legend columns=3},
    nodes near coords,
    nodes near coords align={vertical},
    every node near coord/.append style={font=\scriptsize},
    axis y line*=left,
    axis x line*=bottom,
    bar shift auto,
    draw opacity=1
]

\addplot+[ybar, draw=blue!60!black, fill=blue!20] coordinates {(Dataset 1, 0.79) (Dataset 5, 0.84) (Dataset 9, 0.74) (Dataset 13, 0.62)};
\addplot+[ybar, draw=orange!70!black, fill=orange!30] coordinates {(Dataset 1, 0.80) (Dataset 5, 0.84) (Dataset 9, 0.74) (Dataset 13, 0.63)};
\addplot+[ybar, draw=green!50!black, fill=green!30] coordinates {(Dataset 1, 0.98) (Dataset 5, 0.98) (Dataset 9, 0.97) (Dataset 13, 0.98)};

\legend{Minimal, Aggregated, Ours}
\end{axis}
\end{tikzpicture}
    \caption{Accuracy of MLP trained on three feature representations across selected datasets.}
    \label{fig:feature_comp}
\end{figure}

\paragraph{Comparison of feature representations}
To assess the effect of instance-level information in input features, we compare three types of representations: (i) a minimal set of features consisting of only the basic characteristics of each job (``Minimal''), (ii) features extended with instance-level aggregation statistics -- averages, standard deviations, minima, and maxima over all jobs (``Aggregated''), and (iii) our structured representation described in Section~5.1 (``Ours''). Figure~\ref{fig:feature_comp} summarizes model accuracy for each representation on four selected datasets. The results confirm that our feature design consistently improves predictive performance. We believe this is because simply concatenating summary statistics is not sufficient for the neural network to effectively utilize instance-level information, and that this information should instead be integrated into the job-specific features themselves, as achieved by our proposed featurization approach.

\subsection{Comparison to the state-of-the-art}\label{subsec:results}

In this experiment, we evaluate our approach to the $1|\tilde{d}_i|\sum w_i U_i$ problem against the state-of-the-art methods reviewed in Section~\ref{sec:Literature Review}. The comparison covers both solution quality and computational efficiency across a diverse collection of benchmark datasets.

\paragraph{State-of-the-art heuristics}
We evaluate our method against four heuristic approaches: the max-profit relaxation-based heuristic from~\citep{Baptiste2010}, a set of simple rule-based heuristics, and two metaheuristics.

\begin{itemize}
    \item \textbf{Bapt et al.:} a classical scheduling heuristic originally designed for the $1|\tilde{d}_i|\sum w_i U_i$ problem~\citep{Baptiste2010}.

    \item \textbf{Rule-based:} a composite baseline that evaluates three naive rule-based prediction strategies inside our approach: (i) random (each job has a 50\% chance of being early or tardy), (ii) all jobs are early, and (iii) all jobs are tardy. After each step (i)--(iii), a scheduling algorithm (see Section~\ref{subsec:feas_algo}) transforms the predictions into feasible solutions and computes the objective value; the best of the three values is reported.

    \item \textbf{Genetic Algorithm (GA):} adapted from~\citep{Sevaux2003}, which addresses the weighted number of tardy jobs without hard deadlines. Each solution candidate is represented as a permutation of jobs. We adapt the algorithm to our setting as follows: (i) for each permutation, we compute job completion times; (ii) compare each job’s completion time to its due date to assign early/tardy labels (simulating an oracle); (iii) pass the labeled jobs to our scheduling framework (Section~\ref{subsec:feas_algo}), which constructs a feasible schedule and computes its cost. The original method includes a local search component with quadratic complexity, however, we excluded local search from our implementation due to excessive runtime on large instances.

    \item \textbf{Honey Badger:} a recent metaheuristic by~\cite{Hashim2022} which has been shown to perform well across a variety of scheduling problems~\citep{Hassan2024}. In our adaptation, each candidate solution is a vector of real numbers in $[0,1]$, which is rounded to a binary vector representing early/tardy predictions. These predictions are then passed to the same scheduling framework used in ours and the GA approaches.
\end{itemize}

\paragraph{Evaluation criteria and settings}
All methods are evaluated on the fifteen benchmark datasets described in Section~\ref{subsec:datasets}, each comprising 100–200 instances with varying sizes and distributions. 
We use two performance criteria:
\begin{enumerate}
    \item The percentage of instances solved to optimality ($n_{opt}$);
    \item The average optimality gap:
    \begin{equation}
        \Delta_{avg} = \frac{f^* - f(s)}{f^*} \cdot 100\%,
    \end{equation}
    where $f(s)$ and $f^*$ denote the objective value of the constructed and optimal schedules, respectively.
\end{enumerate}

A timeout of 300 seconds is imposed for each heuristic per instance. For our approach, the ILP-based routine in the \texttt{Refine} function (Section \ref{subsec:refine_predictions}) is limited to 60 seconds, although in practice it typically terminates within fractions of a second.

\begin{table*}[htbp]
    \small
    \centering
    \captionsetup{justification=centering}
    \renewcommand{\arraystretch}{1.2}
    \begin{tabularx}{\textwidth}{c c *{6}{>{\centering\arraybackslash}X}}
    \toprule
    \multirow{2}{*}{\textbf{Method}} & \multirow{2}{*}{\textbf{Metric}} & \multicolumn{6}{c}{\textbf{Instance Size}} \\
    \cmidrule(lr){3-8}
    & & \textbf{500} & \textbf{1000} & \textbf{2000} & \textbf{3000} & \textbf{4000} & \textbf{5000} \\
    \midrule
    \multirow{2}{*}{Proposed} 
    & $\Delta_{avg}$ & 0.009 & 0.002 & 0.001 & 0.002 & 0.001 & 0.001 \\
    & $n_{opt}$ & 95 & 95 & 88 & 80 & 72 & 68 \\
    \midrule

    \multirow{2}{*}{Bapt et al} 
    & $\Delta_{avg}$ & 5.1 & 7.1 & 7.2 & 7.7 & 8.8 & 10 \\
    & $n_{opt}$ & 38 & 35 & 40 & 41 & 45 & 46 \\
    \midrule

    \multirow{2}{*}{GA} 
    & $\Delta_{avg}$ & 29.1 & 30.7 & 31.7 & 32.6 & 33.1 & 33.3 \\
    & $n_{opt}$ & 0 & 0 & 0 & 0 & 0 & 0 \\
    \midrule

    \multirow{2}{*}{Honey Badger} 
    & $\Delta_{avg}$ & 23.9 & 25.5 & 26.8 & 27.8 & 28.1 & 28.3 \\
    & $n_{opt}$ & 0 & 0 & 0 & 0 & 0 & 0 \\
    \midrule

    \multirow{2}{*}{Rule-based} 
    & $\Delta_{avg}$ & 32.3 & 32.6 & 32.8 & 32.9 & 32.9 & 32.9 \\
    & $n_{opt}$ & 0 & 0 & 0 & 0 & 0 & 0 \\
    \bottomrule
\end{tabularx}
    \caption{Average number of optimal solutions and optimality gap per instance size (mean over all datasets)}
    \label{tab:results}
\end{table*}

\begin{figure*}[htbp]
    \centering
    \begin{tikzpicture}
    \begin{axis}[
        ybar,
        bar width=7pt,
        symbolic x coords={1, 2, 3, 4, 5, 6, 7, 8, 9, 10, 11, 12, 13, 14, 15},
        xtick=data,
        xticklabel style={rotate=0},
        ylabel={$\Delta_{avg}$ [\%, log scale]},
        xlabel={Dataset ID},
        ymode=log,
        log basis y=10,
        legend style={at={(0.5,1.05)}, anchor=south, legend columns=2},
        width=\textwidth,
        height=6cm,
        ymin=0.00001,
        ymajorgrids=true,
        grid style=dashed
    ]
    \addplot+[fill=blue!50] coordinates {
    (1, 0.000083) (2, 0.002467) (3, 0.013883) (4, 0.002950) (5, 0.003733) (6, 0.000917) (7, 0.000483) (8, 0.004717) (9, 0.000983) (10, 0.000050) (11, 0.011400) (12, 0.000517) (13, 0.000200) (14, 0.000117) (15, 0.000083) };
    
    \addplot+[fill=red!50] coordinates {
    (1, 0.007442) (2, 8.584260) (3, 2.260083) (4, 5.826763) (5, 27.23) (6, 0.168547) (7, 0.091533) (8, 0.003372) (9, 0.678965) (10, 1.000175) (11, 0.741150) (12, 0.872165) (13, 0.071422) (14, 0.071270) (15, 0.136740)};
    \legend{Proposed, Best of Others}
    \end{axis}
\end{tikzpicture}
    \caption{Comparison of average optimality gap (log scale): proposed method vs. best alternative across datasets.}
    \label{fig:opt_gaps}
\end{figure*}
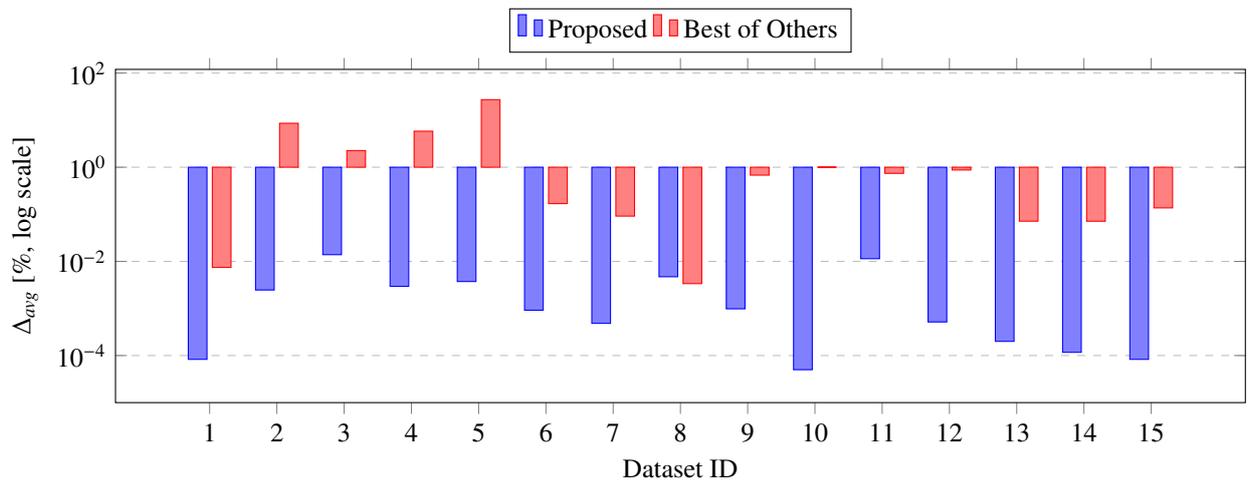

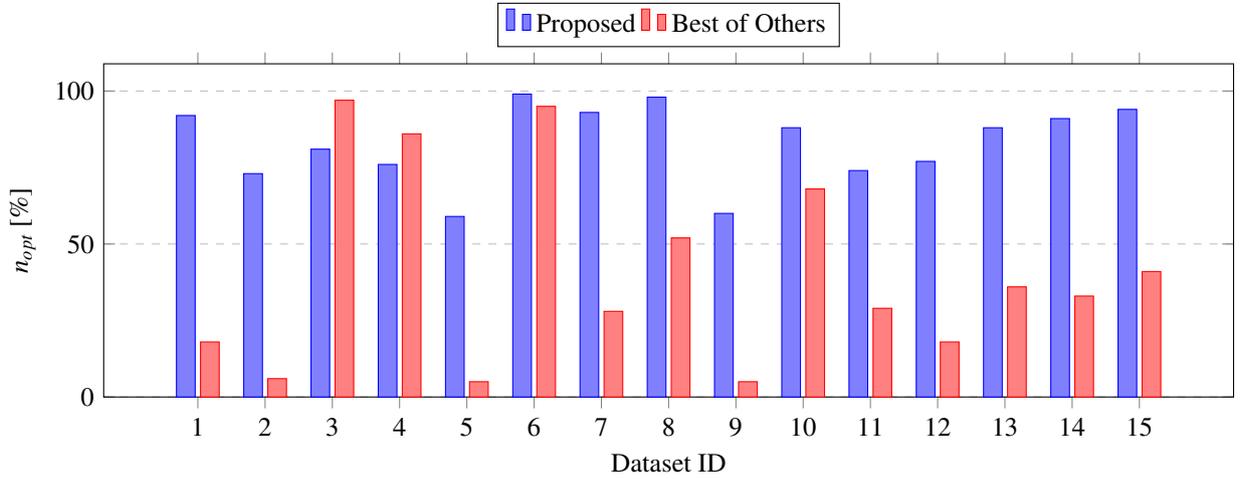
\begin{figure*}[htbp]
    \centering
    \begin{tikzpicture}
    \begin{axis}[
        ybar,
        bar width=7pt,
        symbolic x coords={1, 2, 3, 4, 5, 6, 7, 8, 9, 10, 11, 12, 13, 14, 15},
        xtick=data,
        xticklabel style={rotate=0},
        ylabel={$n_{opt}$ [\%]},
        xlabel={Dataset ID},
        legend style={at={(0.5,1.05)}, anchor=south, legend columns=2},
        width=\textwidth,
        height=6cm,
        ymin=0,
        ymajorgrids=true,
        grid style=dashed
    ]
    \addplot+[fill=blue!50] coordinates {
        (1, 92) (2, 73) (3, 81) (4, 76) (5, 59)
        (6, 99) (7, 93) (8, 98) (9, 60) (10, 88)
        (11, 74) (12, 77) (13, 88) (14, 91) (15, 94)
    };
    \addplot+[fill=red!50] coordinates {
        (1, 18) (2, 6) (3, 97) (4, 86) (5, 5)
        (6, 95) (7, 28) (8, 52) (9, 5) (10, 68)
        (11, 29) (12, 18) (13, 36) (14, 33) (15, 41)
    };
    \legend{Proposed, Best of Others}
    \end{axis}
\end{tikzpicture}
    \caption{Comparison of the average number of optimal solutions: proposed method vs. best alternative across datasets.}
    \label{fig:num_opts}
\end{figure*}

\begin{table}[htbp]
    \centering
    \begin{tabularx}{\textwidth}{|
    >{\centering\arraybackslash}X|
    >{\centering\arraybackslash}X|
    >{\centering\arraybackslash}X|
    >{\centering\arraybackslash}X|
    >{\centering\arraybackslash}X|
    >{\centering\arraybackslash}X|
    >{\centering\arraybackslash}X|} 
    \hline
     & \multicolumn{3}{c|}{Proposed approach $\;$[s]} & \multicolumn{3}{c|}{Baptiste et al $\;$[s]} \\
    \hline
    n & min & avg & max & min & avg & max \\
    \hline
    500 & 0.1 & 0.4 & 1.29 & 0.03 & 15.28 & $>$ 300  \\
    1000 & 0.12 & 0.96 & 1.33 & 0.05 & 21.93 & $>$ 300  \\
    2000 & 0.24 & 3.45 & 4.5 & 0.19 & 22.79 & $>$ 300  \\
    3000 & 0.37 & 7.43 & 9.69 & 0.31 & 24.05 & $>$ 300  \\
    4000 & 0.5 & 13.05 & 17.09 & 0.54 & 28.28 & $>$ 300  \\
    5000 & 0.65 & 20.06 & 26.36 & 0.71 & 32.1 & $>$ 300  \\
    \hline
    \hline
     & \multicolumn{3}{c|}{Metaheuristics (per epoch) $\;$[s]} & \multicolumn{3}{c|}{Rule-based $\;$[s]} \\
    \hline
    n & min & avg & max & min & avg & max \\
    \hline
    500 & 0.32 & 0.43 & 0.61 & 0.02 & 0.03 & 0.03  \\
    1000 & 0.70 & 0.88 & 1.32 & 0.05 & 0.05 & 0.06  \\
    2000 & 1.67 & 2.08 & 2.76 & 0.12 & 0.13 & 0.14  \\
    3000 & 2.95 & 3.22 & 3.95 & 0.2 & 0.22 & 0.24  \\
    4000 & 4.39 & 4.9 & 5.64 & 0.31 & 0.32 & 0.35  \\
    5000 & 6.18 & 6.81 & 7.34 & 0.42 & 0.45 & 0.49  \\
    \hline
\end{tabularx}
    \caption{Runtime comparison of different methods for varying problem sizes.}
    \label{tab:runtimes}
\end{table}

\paragraph{Results Analysis}
The results in Table~\ref{tab:results} show that our method consistently outperforms all the baselines across the introduced datasets. Compared to \emph{Bapt et al.}, our approach achieves optimality gaps nearly three orders of magnitude smaller (always $<$0.01\% vs. 5--10\%) and solves up to 95\% of instances to optimality (vs. 46\% for \emph{Bapt et al.}). The comparison is even more pronounced against metaheuristics (\emph{GA} and \emph{Honey Badger}), which show large gaps (29--33\% and 24--28\%, respectively) and fail to solve any instance optimally. The rule-based heuristic performs the worst, with gaps up to 41\% and no optimal solutions.

Figures~\ref{fig:opt_gaps}–\ref{fig:num_opts} show the average optimality gaps and the number of optimal solutions achieved by our approach, compared to the best of the considered state-of-the-art methods. The comparison is done per dataset, across all instance sizes. In nearly all cases, our method outperforms the best competing heuristic -- typically the approach by \cite{Baptiste2010}. An exception occurs with Dataset~5, where \emph{Bapt et al.}'s gap increases sharply to 94.5\%, making \emph{Honey Badger} (27.23\%) the best among the baselines for that case (see Section~\ref{subsec:further_discussion}).
Our method achieves significantly smaller optimality gaps on 14 of 15 datasets (sometimes the gap smaller in several orders of magnitude, e.g. Datasets 4-7 and 10) and the highest number of optimal solutions on 12 out of 15 datasets. This highlights the robustness of our method across different instance distributions and dataset properties. Even in those datasets where another method achieves slightly higher number of optima (e.g., Datasets~3 and~4), our approach still leads in terms of optimality gap. Conversely, in Dataset~8, our method finds more optima but has a slightly higher average gap. Such trade-offs are expected: failing to find the optimum on a few instances may result in larger gaps, while finding fewer optima but with small gaps on the rest can still yield strong average performance.

Table~\ref{tab:runtimes} provides a comparison of runtime performance. Metaheuristics were run for a sufficient number of epochs to stay within the global 300-second time budget, and reported runtimes correspond to a single epoch. Since the per-epoch runtimes of \emph{GA} and \emph{Honey Badger} are similar, we report them jointly. Our approach achieves a notable runtime performance with a maximal runtime of under 30 seconds. In contrast, \emph{Bapt et al.} occasionally exceeds the time limit. The rule-based heuristic is the fastest across all instance sizes but suffers from poor solution quality, with large gaps and no optimal solutions (see Table~\ref{tab:results}).

In summary, our method consistently achieves low optimality gaps, solves most instances to optimality (typically 80--100\%), and does so with competitive runtimes, demonstrating strong performance across diverse datasets.

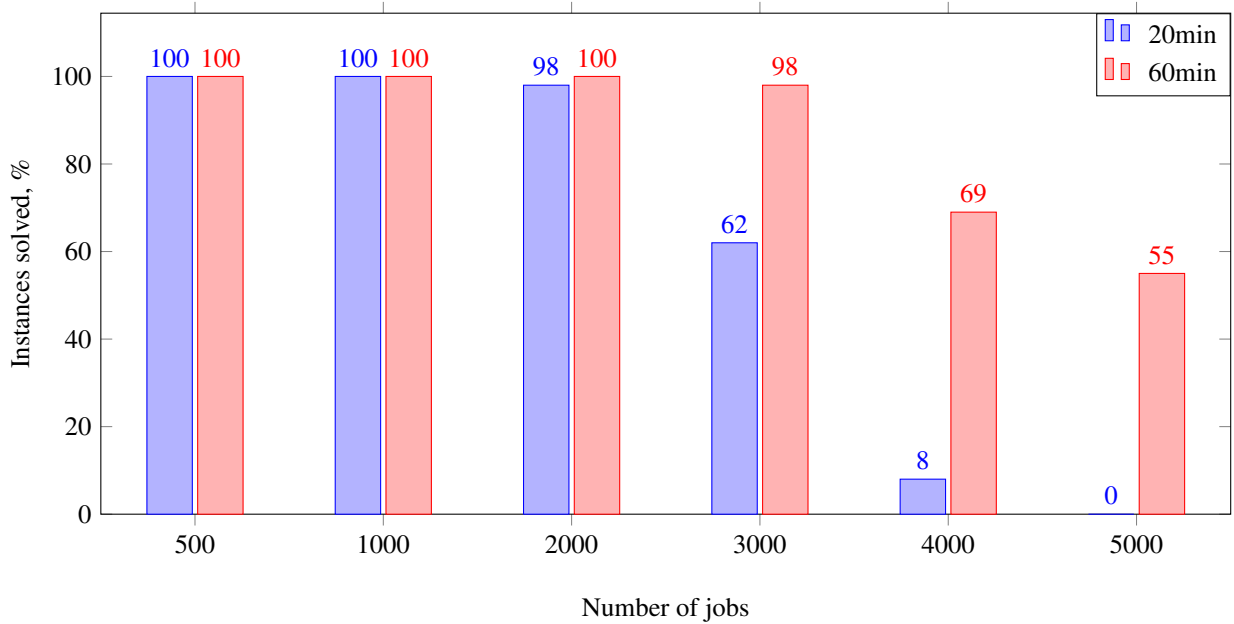
\begin{figure}
    \centering
    \begin{tikzpicture}
\begin{axis}[
    ybar,
    bar width = 17pt,
    height=0.5\textwidth,
    width=\textwidth,
    xlabel={Number of jobs},
    xlabel style={yshift=-10pt},
    ylabel={Instances solved, \%},
    ymin = 0, ymax = 109, xmax=6,
    xtick = data,
    xticklabels={500, 1000, 2000, 3000, 4000, 5000},
    enlarge y limits = {value = .05, upper},
    enlarge x limits = {abs = .5},
    nodes near coords,
    legend style={
        at={(1,1)},
        anchor=north east,
        column sep=0.5em,
        row sep=0.5ex,
        },
    ]
    
\addplot coordinates {(1,100) (2,100) (3,98) (4,62) (5,8) (6,0)};
\addplot coordinates {(1,100) (2,100) (3,100) (4,98) (5,69) (6,55)};
\legend {20min, 60min};

\end{axis}
\end{tikzpicture}
    \caption{Percentage of Dataset 5 instances solved within 20-minute and 60-minute timeouts by the algorithm from~\citep{Baptiste2010}}
    \label{fig:discussion}
\end{figure}

\subsection{Empirical Validation of the Instability in the Exact Approach}\label{subsec:further_discussion}

This section provides empirical evidence highlighting the instability of the exact algorithm proposed by~\cite{Baptiste2010}. 
While the algorithm typically solves most instances within a 20-minute time limit, there are two notable exceptions: the algorithm fails to solve approximately 5\% instances with 5000 jobs from Datasets 2 and 4 within the standard 20-minute limit, though extending the time limit to 60 minutes resolves the majority of them; (ii) instances from the Dataset 5 pose a significant challenge with numerous instances remaining unsolved even after the 60-minute timeout. Figure~\ref{fig:discussion} illustrates the percentage of Dataset 5 instances solved under two time limits.
These results emphasize the limitations of Baptiste’s approach, particularly for Dataset 5, where both the exact algorithm and the heuristic perform poorly. 
In such cases, our data-driven approach stands out as the only practical solution for producing high-quality results in less than half a minute.

\begin{figure}[t]
    \centering
    \begin{tikzpicture}
\begin{axis}[
    width=12cm,
    height=8cm,
    xlabel={Accuracy (\%)},
    ylabel={$\Delta_{avg}$ [\%]},
    xmin=50, xmax=100,
    ymin=0, ymax=50,
    grid=both,
    legend style={at={(0.5,-0.25)}, anchor=north, legend columns=4},
]

\addplot[
    only marks,
    mark=*,
    color=red,
    mark size=2pt,
] coordinates {
    (70,32) (71,35) (72,37) (73,35) (74,35) (75,28) (76,32) (79,23) (78,37)
};
\addlegendentry{70–79\%}

\addplot[
    only marks,
    mark=square*,
    color=blue,
    mark size=2pt,
] coordinates {
    (80,25) (81,20) (82,23) (83,25) (84,20) (85,23) (86,21) (87,19) (88,18) (89,13)
};
\addlegendentry{80–89\%}

\addplot[
    only marks,
    mark=triangle*,
    color=green!60!black,
    mark size=2.5pt,
] coordinates {
    (90,7) (91,4) (92,4) (93,2) (94,2)
};
\addlegendentry{90–94\%}

\addplot[
    only marks,
    mark=star,
    color=purple,
    mark size=3pt,
] coordinates {
    (95,0.5) (96,0.1) (97,0.01) (98,0.01) (99,0.01)
};
\addlegendentry{95–99\%}

\end{axis}
\end{tikzpicture}
    \caption{Relationship between training accuracy of the model and the average optimality gap.}
    \label{fig:acc_gap_sensitivity}
\end{figure}
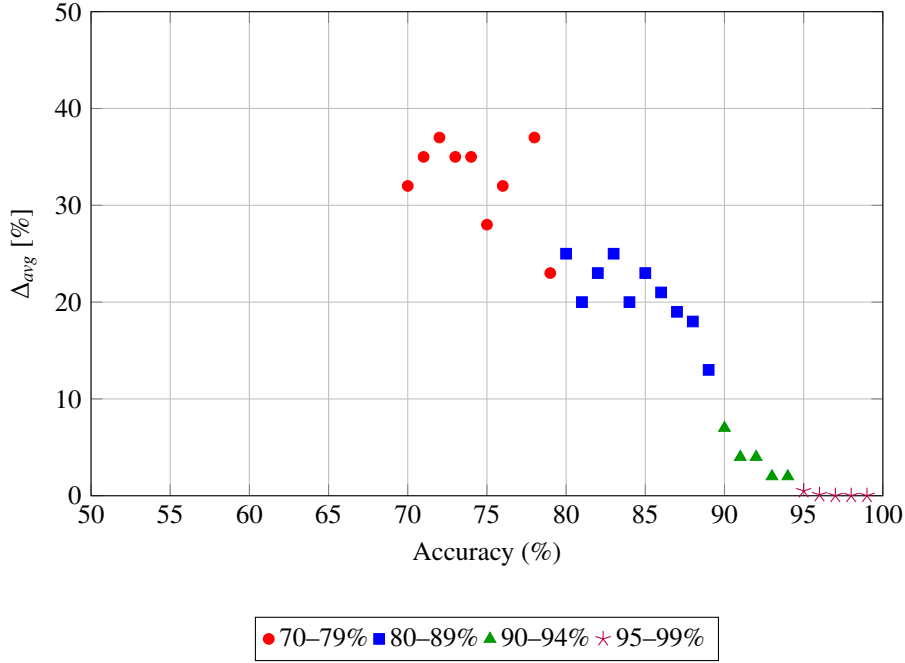

\subsection{Impact of model accuracy on scheduling performance}

To evaluate whether the accuracy of the machine learning model correlates with the final performance of the scheduling pipeline, we conducted an additional experiment. For each dataset and model type (MLP, conditional probabilities, attention-based), we trained multiple models with varying levels of accuracy by modifying the number of training epochs and regularization parameters.
Each of these models was then used in the full scheduling pipeline to generate early/tardy predictions and compute the resulting optimality gap. The results are summarized in Figure~\ref{fig:acc_gap_sensitivity}, where we report the average optimality gap grouped by training accuracy ranges (70--79\%, 80--89\%, 90--94\%, and 95--100\%).
The figure shows a clear trend across all model types: as the training accuracy increases, the resulting optimality gap consistently decreases. This indicates that training accuracy serves as a good proxy for scheduling performance, allowing us to assess the quality of a model without having to run the full scheduling pipeline.

\section{Conclusions and Future Work}\label{sec:Conclusion}
This paper presents a novel data-driven approach to address the fundamental scheduling problem of minimizing the weighted number of tardy jobs on a single machine. 
By integrating machine learning with problem-specific properties, our approach consistently outperforms state-of-the-art heuristics, providing better quality solutions within the same time limits. 
We comprehensively study various ML models and select one with strong generalization capabilities across diverse data distributions. 
Our approach constructs a feasible solution for every instance that has at least one, addressing a common challenge in incorporating ML into combinatorial optimization problems. 
Additionally, our model showcases wide applicability across different data distributions, surpassing standard practices in operational research. 
Overall, our proposed method offers a significant improvement over existing approaches, contributing to the advancement of scheduling algorithms in practical domains.

As a direction for future research, the considered problem can be extended with release times, e.g., $1|r_j, \Tilde{d}_j|\sum w_jU_j$. 
This step significantly increases the complexity, as it becomes NP-hard even to find a feasible solution. 
Alternatively, the research can be extended to the parallel machines scheduling problem $P|\pi_j \sim \mathcal{N}(\mu_j, \sigma_j^2)|Pr(C_{max} \leq \delta)$, as outlined in the recent study by~\cite{Novak2022}. 
This problem is strongly NP-hard, and currently, it relies on a genetic algorithm-based heuristic for an initial solution, followed by a refinement process. 
However, a data-driven approach could offer a promising alternative. Notably, the problem assumes normally distributed processing times, which could facilitate the application of machine learning techniques, as the well-defined distribution enhances the potential for effective learning.

\section*{Acknowledgements}
This work was funded by the Czech Ministry of Education, Youth and Sports under the ERC~CZ project POSTMAN no.~LL1902, by the European Union under the project ROBOPROX (reg.\ no.\ CZ.02.01.01/00/22\_008/0004590), by the Grant Agency of the Czech Republic under the Project GACR 22-31670S and by the Grant Agency of the Czech Technical University in Prague, grant No. SGS25/144/OHK3/3T/13. \emph{Personal note (Nikolai Antonov): To my mom and granny, my beloved Sophia, and my friend Alexander.}


\bibliographystyle{elsarticle-harv} 
\bibliography{references}

\begin{thebibliography}{60}
\expandafter\ifx\csname natexlab\endcsname\relax\def\natexlab#1{#1}\fi
\providecommand{\url}[1]{\texttt{#1}}
\providecommand{\href}[2]{#2}
\providecommand{\path}[1]{#1}
\providecommand{\DOIprefix}{doi:}
\providecommand{\ArXivprefix}{arXiv:}
\providecommand{\URLprefix}{URL: }
\providecommand{\Pubmedprefix}{pmid:}
\providecommand{\doi}[1]{\href{http://dx.doi.org/#1}{\path{#1}}}
\providecommand{\Pubmed}[1]{\href{pmid:#1}{\path{#1}}}
\providecommand{\bibinfo}[2]{#2}
\ifx\xfnm\relax \def\xfnm[#1]{\unskip,\space#1}\fi
\bibitem[{Abdelghany et~al.(2024)Abdelghany, Abdelghany and Guzhva}]{Abdelghany2024}
\bibinfo{author}{Abdelghany, A.}, \bibinfo{author}{Abdelghany, K.}, \bibinfo{author}{Guzhva, V.S.}, \bibinfo{year}{2024}.
\newblock \bibinfo{title}{Schedule-level optimization of flight block times for improved airline schedule planning: A data-driven approach}.
\newblock \bibinfo{journal}{Journal of Air Transport Management} \bibinfo{volume}{115}, \bibinfo{pages}{102535}.
\newblock \DOIprefix\doi{https://doi.org/10.1016/j.jairtraman.2023.102535}.
\bibitem[{Adamu and Adewumi(2014)}]{Muminu2014}
\bibinfo{author}{Adamu, M.O.}, \bibinfo{author}{Adewumi, A.O.}, \bibinfo{year}{2014}.
\newblock \bibinfo{title}{A survey of single machine scheduling to minimize weighted number of tardy jobs}.
\newblock \bibinfo{journal}{Journal of Industrial and Management Optimization} \bibinfo{volume}{10}, \bibinfo{pages}{219--241}.
\newblock \DOIprefix\doi{10.3934/jimo.2014.10.219}.
\bibitem[{Alicastro et~al.(2021)Alicastro, Ferone, Festa, Fugaro and Pastore}]{Alicastro2021}
\bibinfo{author}{Alicastro, M.}, \bibinfo{author}{Ferone, D.}, \bibinfo{author}{Festa, P.}, \bibinfo{author}{Fugaro, S.}, \bibinfo{author}{Pastore, T.}, \bibinfo{year}{2021}.
\newblock \bibinfo{title}{A reinforcement learning iterated local search for makespan minimization in additive manufacturing machine scheduling problems}.
\newblock \bibinfo{journal}{Computers \& Operations Research} \bibinfo{volume}{131}, \bibinfo{pages}{105272}.
\newblock \DOIprefix\doi{https://doi.org/10.1016/j.cor.2021.105272}.
\bibitem[{Amro et~al.(2021)Amro, Al-Akhras, Hindi, Habib and Shawar}]{Amro2021}
\bibinfo{author}{Amro, A.}, \bibinfo{author}{Al-Akhras, M.}, \bibinfo{author}{Hindi, K.E.}, \bibinfo{author}{Habib, M.}, \bibinfo{author}{Shawar, B.A.}, \bibinfo{year}{2021}.
\newblock \bibinfo{title}{Instance reduction for avoiding overfitting in decision trees}.
\newblock \bibinfo{journal}{Journal of Intelligent Systems} \bibinfo{volume}{30}, \bibinfo{pages}{438--459}.
\newblock \URLprefix \url{https://doi.org/10.1515/jisys-2020-0061}, \DOIprefix\doi{doi:10.1515/jisys-2020-0061}.
\bibitem[{Antonov et~al.(2023)Antonov, {\v{S}}ucha and Janota}]{Antonov2023}
\bibinfo{author}{Antonov, N.}, \bibinfo{author}{{\v{S}}ucha, P.}, \bibinfo{author}{Janota, M.}, \bibinfo{year}{2023}.
\newblock \bibinfo{title}{Data-driven single machine scheduling minimizing weighted number of tardy jobs}, in: \bibinfo{editor}{Moniz, N.}, \bibinfo{editor}{Vale, Z.}, \bibinfo{editor}{Cascalho, J.}, \bibinfo{editor}{Silva, C.}, \bibinfo{editor}{Sebasti{\~a}o, R.} (Eds.), \bibinfo{booktitle}{Progress in Artificial Intelligence}, \bibinfo{publisher}{Springer Nature Switzerland}, \bibinfo{address}{Cham}. pp. \bibinfo{pages}{483--494}.
\bibitem[{Awada et~al.(2021)Awada, Srour and Srour}]{Awada2021}
\bibinfo{author}{Awada, M.}, \bibinfo{author}{Srour, F.J.}, \bibinfo{author}{Srour, I.M.}, \bibinfo{year}{2021}.
\newblock \bibinfo{title}{Data-driven machine learning approach to integrate field submittals in project scheduling}.
\newblock \bibinfo{journal}{Journal of Management in Engineering} \bibinfo{volume}{37}.
\newblock \DOIprefix\doi{10.1061/(asce)me.1943-5479.0000873}.
\bibitem[{Baptiste et~al.(2010)Baptiste, Croce, Grosso and T'kindt}]{Baptiste2010}
\bibinfo{author}{Baptiste, P.}, \bibinfo{author}{Croce, F.D.}, \bibinfo{author}{Grosso, A.}, \bibinfo{author}{T'kindt, V.}, \bibinfo{year}{2010}.
\newblock \bibinfo{title}{Sequencing a single machine with due dates and deadlines: an {ILP}-based approach to solve very large instances}.
\newblock \bibinfo{journal}{J. Sched.} \bibinfo{volume}{13}, \bibinfo{pages}{39--47}.
\bibitem[{Bengio et~al.(2021)Bengio, Lodi and Prouvost}]{Bengio2021}
\bibinfo{author}{Bengio, Y.}, \bibinfo{author}{Lodi, A.}, \bibinfo{author}{Prouvost, A.}, \bibinfo{year}{2021}.
\newblock \bibinfo{title}{Machine learning for combinatorial optimization: {A} methodological tour d'horizon}.
\newblock \bibinfo{journal}{Eur. J. Oper. Res.} \bibinfo{volume}{290}, \bibinfo{pages}{405--421}.
\bibitem[{Bou\v{s}ka et~al.(2022)Bou\v{s}ka, \v{S}\r{u}cha, Nov\'{a}k and Hanz\'{a}lek}]{Bouska2022}
\bibinfo{author}{Bou\v{s}ka, M.}, \bibinfo{author}{\v{S}\r{u}cha, P.}, \bibinfo{author}{Nov\'{a}k, A.}, \bibinfo{author}{Hanz\'{a}lek, Z.}, \bibinfo{year}{2022}.
\newblock \bibinfo{title}{Deep learning-driven scheduling algorithm for a single machine problem minimizing the total tardiness}.
\newblock \bibinfo{journal}{European Journal of Operational Research} .
\bibitem[{Brammer et~al.(2022)Brammer, Lutz and Neumann}]{Brammer2022}
\bibinfo{author}{Brammer, J.}, \bibinfo{author}{Lutz, B.}, \bibinfo{author}{Neumann, D.}, \bibinfo{year}{2022}.
\newblock \bibinfo{title}{Permutation flow shop scheduling with multiple lines and demand plans using reinforcement learning}.
\newblock \bibinfo{journal}{European Journal of Operational Research} \bibinfo{volume}{299}, \bibinfo{pages}{75--86}.
\newblock \DOIprefix\doi{https://doi.org/10.1016/j.ejor.2021.08.007}.
\bibitem[{Chu et~al.(2023)Chu, Li, Gao, Cui, Pfeiffer and Cui}]{Chu2023}
\bibinfo{author}{Chu, X.}, \bibinfo{author}{Li, S.}, \bibinfo{author}{Gao, F.}, \bibinfo{author}{Cui, C.}, \bibinfo{author}{Pfeiffer, F.}, \bibinfo{author}{Cui, J.}, \bibinfo{year}{2023}.
\newblock \bibinfo{title}{A data-driven meta-learning recommendation model for multi-mode resource constrained project scheduling problem}.
\newblock \bibinfo{journal}{Computers \&amp; Operations Research} \bibinfo{volume}{157}, \bibinfo{pages}{106290}.
\newblock \URLprefix \url{http://dx.doi.org/10.1016/j.cor.2023.106290}, \DOIprefix\doi{10.1016/j.cor.2023.106290}.
\bibitem[{Delgoshaei and Gomes(2016)}]{Delgoshaei2016}
\bibinfo{author}{Delgoshaei, A.}, \bibinfo{author}{Gomes, C.}, \bibinfo{year}{2016}.
\newblock \bibinfo{title}{A multi-layer perceptron for scheduling cellular manufacturing systems in the presence of unreliable machines and uncertain cost}.
\newblock \bibinfo{journal}{Applied Soft Computing} \bibinfo{volume}{49}, \bibinfo{pages}{27--55}.
\newblock \DOIprefix\doi{https://doi.org/10.1016/j.asoc.2016.06.025}.
\bibitem[{Dias and Ierapetritou(2019)}]{Dias2019}
\bibinfo{author}{Dias, L.S.}, \bibinfo{author}{Ierapetritou, M.G.}, \bibinfo{year}{2019}.
\newblock \bibinfo{title}{Data-driven feasibility analysis for the integration of planning and scheduling problems}.
\newblock \bibinfo{journal}{Optimization and Engineering} \bibinfo{volume}{20}, \bibinfo{pages}{1029–1066}.
\newblock \DOIprefix\doi{10.1007/s11081-019-09459-w}.
\bibitem[{Du et~al.(2024)Du, Xie, Liao, Chen, Wu and Xu}]{Du2024}
\bibinfo{author}{Du, Y.}, \bibinfo{author}{Xie, L.}, \bibinfo{author}{Liao, S.}, \bibinfo{author}{Chen, S.}, \bibinfo{author}{Wu, Y.}, \bibinfo{author}{Xu, H.}, \bibinfo{year}{2024}.
\newblock \bibinfo{title}{Dtsmla: A dynamic task scheduling multi-level attention model for stock ranking}.
\newblock \bibinfo{journal}{Expert Systems with Applications} \bibinfo{volume}{243}, \bibinfo{pages}{122956}.
\newblock \DOIprefix\doi{https://doi.org/10.1016/j.eswa.2023.122956}.
\bibitem[{Erickson et~al.(2020)Erickson, Mueller, Shirkov, Zhang, Larroy, Li and Smola}]{Erickson2020}
\bibinfo{author}{Erickson, N.}, \bibinfo{author}{Mueller, J.}, \bibinfo{author}{Shirkov, A.}, \bibinfo{author}{Zhang, H.}, \bibinfo{author}{Larroy, P.}, \bibinfo{author}{Li, M.}, \bibinfo{author}{Smola, A.J.}, \bibinfo{year}{2020}.
\newblock \bibinfo{title}{{AutoGluon}-tabular: Robust and accurate {AutoML} for structured data}.
\newblock \bibinfo{journal}{CoRR} \bibinfo{volume}{abs/2003.06505}.
\newblock \href{http://arxiv.org/abs/2003.06505}{{\tt arXiv:2003.06505}}.
\bibitem[{Feurer et~al.(2020)Feurer, Eggensperger, Falkner, Lindauer and Hutter}]{Feurer2020}
\bibinfo{author}{Feurer, M.}, \bibinfo{author}{Eggensperger, K.}, \bibinfo{author}{Falkner, S.}, \bibinfo{author}{Lindauer, M.}, \bibinfo{author}{Hutter, F.}, \bibinfo{year}{2020}.
\newblock \bibinfo{title}{Auto-sklearn 2.0: Hands-free automl via meta-learning}.
\newblock \DOIprefix\doi{10.48550/ARXIV.2007.04074}.
\bibitem[{Franz(1989)}]{Franz1989}
\bibinfo{author}{Franz, L.S.}, \bibinfo{year}{1989}.
\newblock \bibinfo{title}{Data driven modeling: An application in scheduling}.
\newblock \bibinfo{journal}{Decision Sciences} \bibinfo{volume}{20}, \bibinfo{pages}{359–377}.
\newblock \DOIprefix\doi{10.1111/j.1540-5915.1989.tb01884.x}.
\bibitem[{Graham et~al.(1979)Graham, Lawler, Lenstra and Kan}]{Graham1979}
\bibinfo{author}{Graham, R.}, \bibinfo{author}{Lawler, E.}, \bibinfo{author}{Lenstra, J.}, \bibinfo{author}{Kan, A.}, \bibinfo{year}{1979}.
\newblock \bibinfo{title}{Optimization and approximation in deterministic sequencing and scheduling: a survey}, in: \bibinfo{editor}{Hammer, P.}, \bibinfo{editor}{Johnson, E.}, \bibinfo{editor}{Korte, B.} (Eds.), \bibinfo{booktitle}{Discrete Optimization II}. \bibinfo{publisher}{Elsevier}. volume~\bibinfo{volume}{5} of \textit{\bibinfo{series}{Annals of Discrete Mathematics}}, pp. \bibinfo{pages}{287--326}.
\bibitem[{Hariri and Potts(1994)}]{Hariri1994}
\bibinfo{author}{Hariri, A.M.A.}, \bibinfo{author}{Potts, C.N.}, \bibinfo{year}{1994}.
\newblock \bibinfo{title}{Single machine scheduling with deadlines to minimize the weighted number of tardy jobs}.
\newblock \bibinfo{journal}{Management Science} \bibinfo{volume}{40}, \bibinfo{pages}{1712--1719}.
\bibitem[{Hashim et~al.(2022)Hashim, Houssein, Hussain, Mabrouk and Al-Atabany}]{Hashim2022}
\bibinfo{author}{Hashim, F.A.}, \bibinfo{author}{Houssein, E.H.}, \bibinfo{author}{Hussain, K.}, \bibinfo{author}{Mabrouk, M.S.}, \bibinfo{author}{Al-Atabany, W.}, \bibinfo{year}{2022}.
\newblock \bibinfo{title}{Honey badger algorithm: New metaheuristic algorithm for solving optimization problems}.
\newblock \bibinfo{journal}{Mathematics and Computers in Simulation} \bibinfo{volume}{192}, \bibinfo{pages}{84–110}.
\newblock \URLprefix \url{http://dx.doi.org/10.1016/j.matcom.2021.08.013}, \DOIprefix\doi{10.1016/j.matcom.2021.08.013}.
\bibitem[{Hassan et~al.(2024)Hassan, Abdullahi, Isuwa, Yusuf and Aliyu}]{Hassan2024}
\bibinfo{author}{Hassan, I.H.}, \bibinfo{author}{Abdullahi, M.}, \bibinfo{author}{Isuwa, J.}, \bibinfo{author}{Yusuf, S.A.}, \bibinfo{author}{Aliyu, I.T.}, \bibinfo{year}{2024}.
\newblock \bibinfo{title}{A comprehensive survey of honey badger optimization algorithm and meta-analysis of its variants and applications}.
\newblock \bibinfo{journal}{Franklin Open} \bibinfo{volume}{8}, \bibinfo{pages}{100141}.
\newblock \URLprefix \url{http://dx.doi.org/10.1016/j.fraope.2024.100141}, \DOIprefix\doi{10.1016/j.fraope.2024.100141}.
\bibitem[{Heger and Voss(2021)}]{Heger2021}
\bibinfo{author}{Heger, J.}, \bibinfo{author}{Voss, T.}, \bibinfo{year}{2021}.
\newblock \bibinfo{title}{Dynamically adjusting the k-values of the atcs rule in a flexible flow shop scenario with reinforcement learning}.
\newblock \bibinfo{journal}{International Journal of Production Research} \bibinfo{volume}{61}, \bibinfo{pages}{147–161}.
\newblock \URLprefix \url{http://dx.doi.org/10.1080/00207543.2021.1943762}, \DOIprefix\doi{10.1080/00207543.2021.1943762}.
\bibitem[{Hejl et~al.(2022)Hejl, \v{S}\r{u}cha, Nov{\'{a}}k and Hanz{\'{a}}lek}]{Hejl2022}
\bibinfo{author}{Hejl, L.}, \bibinfo{author}{\v{S}\r{u}cha, P.}, \bibinfo{author}{Nov{\'{a}}k, A.}, \bibinfo{author}{Hanz{\'{a}}lek, Z.}, \bibinfo{year}{2022}.
\newblock \bibinfo{title}{Minimizing the weighted number of tardy jobs on a single machine: Strongly correlated instances}.
\newblock \bibinfo{journal}{Eur. J. Oper. Res.} \bibinfo{volume}{298}, \bibinfo{pages}{413--424}.
\bibitem[{Hermelin et~al.(2024)Hermelin, Molter and Shabtay}]{Hermelin2024}
\bibinfo{author}{Hermelin, D.}, \bibinfo{author}{Molter, H.}, \bibinfo{author}{Shabtay, D.}, \bibinfo{year}{2024}.
\newblock \bibinfo{title}{Minimizing the weighted number of tardy jobs via (max, +)-convolutions}.
\newblock \bibinfo{journal}{INFORMS Journal on Computing} \bibinfo{volume}{36}, \bibinfo{pages}{836–848}.
\newblock \URLprefix \url{http://dx.doi.org/10.1287/ijoc.2022.0307}, \DOIprefix\doi{10.1287/ijoc.2022.0307}.
\bibitem[{Hollmann et~al.(2023)Hollmann, M{\"{u}}ller, Eggensperger and Hutter}]{Hollman2023}
\bibinfo{author}{Hollmann, N.}, \bibinfo{author}{M{\"{u}}ller, S.}, \bibinfo{author}{Eggensperger, K.}, \bibinfo{author}{Hutter, F.}, \bibinfo{year}{2023}.
\newblock \bibinfo{title}{{TabPFNP}: A transformer that solves small tabular classification problems in a second}, in: \bibinfo{booktitle}{The Eleventh International Conference on Learning Representations, {ICLR} 2023, Kigali, Rwanda, May 1-5, 2023}, \bibinfo{publisher}{OpenReview.net}.
\newblock \URLprefix \url{https://openreview.net/pdf?id=cp5PvcI6w8\_}.
\bibitem[{Janssens et~al.(2006)Janssens, Wets, Brijs, Vanhoof, Arentze and Timmermans}]{Janssens2006}
\bibinfo{author}{Janssens, D.}, \bibinfo{author}{Wets, G.}, \bibinfo{author}{Brijs, T.}, \bibinfo{author}{Vanhoof, K.}, \bibinfo{author}{Arentze, T.}, \bibinfo{author}{Timmermans, H.}, \bibinfo{year}{2006}.
\newblock \bibinfo{title}{Integrating bayesian networks and decision trees in a sequential rule-based transportation model}.
\newblock \bibinfo{journal}{European Journal of Operational Research} \bibinfo{volume}{175}, \bibinfo{pages}{16--34}.
\newblock \DOIprefix\doi{https://doi.org/10.1016/j.ejor.2005.03.022}.
\bibitem[{Jun and Lee(2020)}]{Jun2020}
\bibinfo{author}{Jun, S.}, \bibinfo{author}{Lee, S.}, \bibinfo{year}{2020}.
\newblock \bibinfo{title}{Learning dispatching rules for single machine scheduling with dynamic arrivals based on decision trees and feature construction}.
\newblock \bibinfo{journal}{International Journal of Production Research} \bibinfo{volume}{59}, \bibinfo{pages}{2838–2856}.
\newblock \DOIprefix\doi{10.1080/00207543.2020.1741716}.
\bibitem[{Kaandorp and Koole(2007)}]{Kaandorp2007}
\bibinfo{author}{Kaandorp, G.C.}, \bibinfo{author}{Koole, G.}, \bibinfo{year}{2007}.
\newblock \bibinfo{title}{Optimal outpatient appointment scheduling}.
\newblock \bibinfo{journal}{Health Care Management Science} \bibinfo{volume}{10}, \bibinfo{pages}{217–229}.
\newblock \DOIprefix\doi{10.1007/s10729-007-9015-x}.
\bibitem[{Kouteck{\'{a}} et~al.(2024)Kouteck{\'{a}}, Sucha, Hula and Maenhout}]{Koutecka2024}
\bibinfo{author}{Kouteck{\'{a}}, P.}, \bibinfo{author}{Sucha, P.}, \bibinfo{author}{Hula, J.}, \bibinfo{author}{Maenhout, B.}, \bibinfo{year}{2024}.
\newblock \bibinfo{title}{A machine learning approach to rank pricing problems in branch-and-price}.
\newblock \bibinfo{journal}{Eur. J. Oper. Res.} \bibinfo{volume}{320}, \bibinfo{pages}{328--342}.
\newblock \URLprefix \url{https://doi.org/10.1016/j.ejor.2024.07.029}, \DOIprefix\doi{10.1016/J.EJOR.2024.07.029}.
\bibitem[{Lawler(1983)}]{Lawler1983}
\bibinfo{author}{Lawler, E.L.}, \bibinfo{year}{1983}.
\newblock \bibinfo{title}{Scheduling a Single Machine to Minimize the Number of Late Jobs}.
\newblock \bibinfo{type}{Technical Report} \bibinfo{number}{UCB/CSD-83-139}. EECS Department, University of California, Berkeley.
\bibitem[{Lee and Kuiper(2024)}]{Lee2024}
\bibinfo{author}{Lee, R.H.}, \bibinfo{author}{Kuiper, A.}, \bibinfo{year}{2024}.
\newblock \bibinfo{title}{Optimal sequencing using a scheduling heuristic}.
\newblock \bibinfo{journal}{Computers \& Operations Research} \bibinfo{volume}{161}, \bibinfo{pages}{106405}.
\newblock \DOIprefix\doi{https://doi.org/10.1016/j.cor.2023.106405}.
\bibitem[{Liao et~al.(2019)Liao, Zhang, Xia, Chen, Li and Liang}]{Liao2019}
\bibinfo{author}{Liao, Q.}, \bibinfo{author}{Zhang, H.}, \bibinfo{author}{Xia, T.}, \bibinfo{author}{Chen, Q.}, \bibinfo{author}{Li, Z.}, \bibinfo{author}{Liang, Y.}, \bibinfo{year}{2019}.
\newblock \bibinfo{title}{A data-driven method for pipeline scheduling optimization}.
\newblock \bibinfo{journal}{Chemical Engineering Research and Design} \bibinfo{volume}{144}, \bibinfo{pages}{79--94}.
\newblock \DOIprefix\doi{https://doi.org/10.1016/j.cherd.2019.01.017}.
\bibitem[{Liu et~al.(2023)Liu, Piplani and Toro}]{LiuPiplaniToro2023}
\bibinfo{author}{Liu, R.}, \bibinfo{author}{Piplani, R.}, \bibinfo{author}{Toro, C.}, \bibinfo{year}{2023}.
\newblock \bibinfo{title}{A deep multi-agent reinforcement learning approach to solve dynamic job shop scheduling problem}.
\newblock \bibinfo{journal}{Computers \& Operations Research} \bibinfo{volume}{159}, \bibinfo{pages}{106294}.
\newblock \DOIprefix\doi{https://doi.org/10.1016/j.cor.2023.106294}.
\bibitem[{M\'{o}dos et~al.(2016)M\'{o}dos, \v{S}\r{u}cha, V\'{a}clav\'{\i}k, Smejkal and Hanz\'{a}lek}]{Modos2016}
\bibinfo{author}{M\'{o}dos, I.}, \bibinfo{author}{\v{S}\r{u}cha, P.}, \bibinfo{author}{V\'{a}clav\'{\i}k, R.}, \bibinfo{author}{Smejkal, J.}, \bibinfo{author}{Hanz\'{a}lek, Z.}, \bibinfo{year}{2016}.
\newblock \bibinfo{title}{Adaptive online scheduling of tasks with anytime property on heterogeneous resources}.
\newblock \bibinfo{journal}{Computers \&; Operations Research} \bibinfo{volume}{76}, \bibinfo{pages}{95–117}.
\newblock \URLprefix \url{http://dx.doi.org/10.1016/j.cor.2016.06.008}, \DOIprefix\doi{10.1016/j.cor.2016.06.008}.
\bibitem[{Monaci et~al.(2024)Monaci, Agasucci and Grani}]{Monaci2024}
\bibinfo{author}{Monaci, M.}, \bibinfo{author}{Agasucci, V.}, \bibinfo{author}{Grani, G.}, \bibinfo{year}{2024}.
\newblock \bibinfo{title}{An actor-critic algorithm with policy gradients to solve the job shop scheduling problem using deep double recurrent agents}.
\newblock \bibinfo{journal}{European Journal of Operational Research} \bibinfo{volume}{312}, \bibinfo{pages}{910--926}.
\newblock \DOIprefix\doi{https://doi.org/10.1016/j.ejor.2023.07.037}.
\bibitem[{M\"{u}ller et~al.(2022)M\"{u}ller, M\"{u}ller, Kress and Pesch}]{Muller2022}
\bibinfo{author}{M\"{u}ller, D.}, \bibinfo{author}{M\"{u}ller, M.G.}, \bibinfo{author}{Kress, D.}, \bibinfo{author}{Pesch, E.}, \bibinfo{year}{2022}.
\newblock \bibinfo{title}{An algorithm selection approach for the flexible job shop scheduling problem: Choosing constraint programming solvers through machine learning}.
\newblock \bibinfo{journal}{European Journal of Operational Research} \bibinfo{volume}{302}, \bibinfo{pages}{874--891}.
\newblock \DOIprefix\doi{https://doi.org/10.1016/j.ejor.2022.01.034}.
\bibitem[{Nov\'ak et~al.(2022)Nov\'ak, Sucha, Novotny, Stec and Hanzalek}]{Novak2022}
\bibinfo{author}{Nov\'ak, A.}, \bibinfo{author}{Sucha, P.}, \bibinfo{author}{Novotny, M.}, \bibinfo{author}{Stec, R.}, \bibinfo{author}{Hanzalek, Z.}, \bibinfo{year}{2022}.
\newblock \bibinfo{title}{Scheduling jobs with normally distributed processing times on parallel machines}.
\newblock \bibinfo{journal}{European Journal of Operational Research} \bibinfo{volume}{297}, \bibinfo{pages}{422--441}.
\newblock \DOIprefix\doi{https://doi.org/10.1016/j.ejor.2021.05.011}.
\bibitem[{Pedregosa et~al.(2011)Pedregosa, Varoquaux, Gramfort, Michel, Thirion, Grisel, Blondel, Prettenhofer, Weiss, Dubourg, Vanderplas, Passos, Cournapeau, Brucher, Perrot and Duchesnay}]{ScikitLearn2011}
\bibinfo{author}{Pedregosa, F.}, \bibinfo{author}{Varoquaux, G.}, \bibinfo{author}{Gramfort, A.}, \bibinfo{author}{Michel, V.}, \bibinfo{author}{Thirion, B.}, \bibinfo{author}{Grisel, O.}, \bibinfo{author}{Blondel, M.}, \bibinfo{author}{Prettenhofer, P.}, \bibinfo{author}{Weiss, R.}, \bibinfo{author}{Dubourg, V.}, \bibinfo{author}{Vanderplas, J.}, \bibinfo{author}{Passos, A.}, \bibinfo{author}{Cournapeau, D.}, \bibinfo{author}{Brucher, M.}, \bibinfo{author}{Perrot, M.}, \bibinfo{author}{Duchesnay, E.}, \bibinfo{year}{2011}.
\newblock \bibinfo{title}{Scikit-learn: Machine learning in {P}ython}.
\newblock \bibinfo{journal}{Journal of Machine Learning Research} \bibinfo{volume}{12}, \bibinfo{pages}{2825--2830}.
\bibitem[{Pinedo(2012)}]{Pinedo2012}
\bibinfo{author}{Pinedo, M.L.}, \bibinfo{year}{2012}.
\newblock \bibinfo{title}{Scheduling. Theory, Algorithms, and Systems}.
\newblock \bibinfo{publisher}{Springer New York, NY}, \bibinfo{address}{233 Spring St, New York, NY USA}.
\bibitem[{Portoleau et~al.(2024)Portoleau, Artigues and Guillaume}]{Portoleau2024}
\bibinfo{author}{Portoleau, T.}, \bibinfo{author}{Artigues, C.}, \bibinfo{author}{Guillaume, R.}, \bibinfo{year}{2024}.
\newblock \bibinfo{title}{Robust decision trees for the multi-mode project scheduling problem with a resource investment objective and uncertain activity duration}.
\newblock \bibinfo{journal}{European Journal of Operational Research} \bibinfo{volume}{312}, \bibinfo{pages}{525--540}.
\newblock \DOIprefix\doi{https://doi.org/10.1016/j.ejor.2023.07.035}.
\bibitem[{Ren et~al.(2023)Ren, Yuan and Jiao}]{Ren2023}
\bibinfo{author}{Ren, L.}, \bibinfo{author}{Yuan, M.}, \bibinfo{author}{Jiao, X.}, \bibinfo{year}{2023}.
\newblock \bibinfo{title}{Electric vehicle charging and discharging scheduling strategy based on dynamic electricity price}.
\newblock \bibinfo{journal}{Engineering Applications of Artificial Intelligence} \bibinfo{volume}{123}, \bibinfo{pages}{106320}.
\newblock \DOIprefix\doi{https://doi.org/10.1016/j.engappai.2023.106320}.
\bibitem[{Rossit et~al.(2019)Rossit, Tohm\'{e} and Frutos}]{Rossit2019}
\bibinfo{author}{Rossit, D.A.}, \bibinfo{author}{Tohm\'{e}, F.}, \bibinfo{author}{Frutos, M.}, \bibinfo{year}{2019}.
\newblock \bibinfo{title}{A data-driven scheduling approach to smart manufacturing}.
\newblock \bibinfo{journal}{Journal of Industrial Information Integration} \bibinfo{volume}{15}, \bibinfo{pages}{69--79}.
\newblock \DOIprefix\doi{https://doi.org/10.1016/j.jii.2019.04.003}.
\bibitem[{{Sadeghi Darvazeh} et~al.(2024){Sadeghi Darvazeh}, {Mansoori Mooseloo}, Gholian-Jouybari, Amiri, Bonakdari and Hajiaghaei-Keshteli}]{Darvazeh2024}
\bibinfo{author}{{Sadeghi Darvazeh}, S.}, \bibinfo{author}{{Mansoori Mooseloo}, F.}, \bibinfo{author}{Gholian-Jouybari, F.}, \bibinfo{author}{Amiri, M.}, \bibinfo{author}{Bonakdari, H.}, \bibinfo{author}{Hajiaghaei-Keshteli, M.}, \bibinfo{year}{2024}.
\newblock \bibinfo{title}{Data-driven robust optimization to design an integrated sustainable forest biomass-to-electricity network under disjunctive uncertainties}.
\newblock \bibinfo{journal}{Applied Energy} \bibinfo{volume}{356}, \bibinfo{pages}{122404}.
\newblock \DOIprefix\doi{https://doi.org/10.1016/j.apenergy.2023.122404}.
\bibitem[{Sarin et~al.(2010)Sarin, Nagarajan and Liao}]{Sarin2010}
\bibinfo{author}{Sarin, S.C.}, \bibinfo{author}{Nagarajan, B.}, \bibinfo{author}{Liao, L.}, \bibinfo{year}{2010}.
\newblock \bibinfo{title}{Stochastic Scheduling: Expectation-Variance Analysis of a Schedule}.
\newblock \bibinfo{publisher}{Cambridge University Press}.
\newblock \URLprefix \url{http://dx.doi.org/10.1017/CBO9780511778032}, \DOIprefix\doi{10.1017/cbo9780511778032}.
\bibitem[{Saxena et~al.(2024)Saxena, Kumar, Rao, Mondloe, Dhapekar, Sharma and Yadav}]{Saxena2024}
\bibinfo{author}{Saxena, N.}, \bibinfo{author}{Kumar, R.}, \bibinfo{author}{Rao, Y.K.S.S.}, \bibinfo{author}{Mondloe, D.S.}, \bibinfo{author}{Dhapekar, N.K.}, \bibinfo{author}{Sharma, A.}, \bibinfo{author}{Yadav, A.S.}, \bibinfo{year}{2024}.
\newblock \bibinfo{title}{Hybrid knn-svm machine learning approach for solar power forecasting}.
\newblock \bibinfo{journal}{Environmental Challenges} \bibinfo{volume}{14}, \bibinfo{pages}{100838}.
\newblock \DOIprefix\doi{https://doi.org/10.1016/j.envc.2024.100838}.
\bibitem[{Sevaux and Dauz\`{e}re-P\'{e}r\`{e}s(2003)}]{Sevaux2003}
\bibinfo{author}{Sevaux, M.}, \bibinfo{author}{Dauz\`{e}re-P\'{e}r\`{e}s, S.}, \bibinfo{year}{2003}.
\newblock \bibinfo{title}{Genetic algorithms to minimize the weighted number of late jobs on a single machine}.
\newblock \bibinfo{journal}{European Journal of Operational Research} \bibinfo{volume}{151}, \bibinfo{pages}{296--306}.
\newblock \DOIprefix\doi{https://doi.org/10.1016/S0377-2217(02)00827-5}. \bibinfo{note}{meta-heuristics in combinatorial optimization}.
\bibitem[{Uzunoglu et~al.(2025)Uzunoglu, Gahm and Tuma}]{Uzunoglu2025}
\bibinfo{author}{Uzunoglu, A.}, \bibinfo{author}{Gahm, C.}, \bibinfo{author}{Tuma, A.}, \bibinfo{year}{2025}.
\newblock \bibinfo{title}{Machine learning based algorithm selection and genetic algorithms for serial-batch scheduling}.
\newblock \bibinfo{journal}{Computers \&; Operations Research} \bibinfo{volume}{173}, \bibinfo{pages}{106827}.
\newblock \DOIprefix\doi{10.1016/j.cor.2024.106827}.
\bibitem[{V\'{a}clav\'{\i}k et~al.(2018)V\'{a}clav\'{\i}k, Nov\'{a}k, \v{S}\r{u}cha and Hanz\'{a}lek}]{Vaclavik2018}
\bibinfo{author}{V\'{a}clav\'{\i}k, R.}, \bibinfo{author}{Nov\'{a}k, A.}, \bibinfo{author}{\v{S}\r{u}cha, P.}, \bibinfo{author}{Hanz\'{a}lek, Z.}, \bibinfo{year}{2018}.
\newblock \bibinfo{title}{Accelerating the branch-and-price algorithm using machine learning}.
\newblock \bibinfo{journal}{European Journal of Operational Research} \bibinfo{volume}{271}, \bibinfo{pages}{1055--1069}.
\newblock \DOIprefix\doi{https://doi.org/10.1016/j.ejor.2018.05.046}.
\bibitem[{{van Essen} et~al.(2012){van Essen}, Hans, Hurink and Oversberg}]{Vanessen2012}
\bibinfo{author}{{van Essen}, J.}, \bibinfo{author}{Hans, E.}, \bibinfo{author}{Hurink, J.}, \bibinfo{author}{Oversberg, A.}, \bibinfo{year}{2012}.
\newblock \bibinfo{title}{Minimizing the waiting time for emergency surgery}.
\newblock \bibinfo{journal}{Operations Research for Health Care} \bibinfo{volume}{1}, \bibinfo{pages}{34--44}.
\newblock \DOIprefix\doi{https://doi.org/10.1016/j.orhc.2012.05.002}.
\bibitem[{Vaswani et~al.(2017)Vaswani, Shazeer, Parmar, Uszkoreit, Jones, Gomez, Kaiser and Polosukhin}]{Attention2017}
\bibinfo{author}{Vaswani, A.}, \bibinfo{author}{Shazeer, N.}, \bibinfo{author}{Parmar, N.}, \bibinfo{author}{Uszkoreit, J.}, \bibinfo{author}{Jones, L.}, \bibinfo{author}{Gomez, A.N.}, \bibinfo{author}{Kaiser, L.}, \bibinfo{author}{Polosukhin, I.}, \bibinfo{year}{2017}.
\newblock \bibinfo{title}{Attention is all you need}, in: \bibinfo{editor}{Guyon, I.}, \bibinfo{editor}{von Luxburg, U.}, \bibinfo{editor}{Bengio, S.}, \bibinfo{editor}{Wallach, H.M.}, \bibinfo{editor}{Fergus, R.}, \bibinfo{editor}{Vishwanathan, S.V.N.}, \bibinfo{editor}{Garnett, R.} (Eds.), \bibinfo{booktitle}{Advances in Neural Information Processing Systems 30: Annual Conference on Neural Information Processing Systems 2017, December 4-9, 2017, Long Beach, CA, {USA}}, pp. \bibinfo{pages}{5998--6008}.
\bibitem[{Wang(1999)}]{Wang1999}
\bibinfo{author}{Wang, P.}, \bibinfo{year}{1999}.
\newblock \bibinfo{title}{Sequencing and scheduling {N} customers for a stochastic server}.
\newblock \bibinfo{journal}{European Journal of Operational Research} \bibinfo{volume}{119}, \bibinfo{pages}{729--738}.
\newblock \DOIprefix\doi{https://doi.org/10.1016/S0377-2217(98)00340-3}.
\bibitem[{Wang and Tang(2017)}]{Wang2017}
\bibinfo{author}{Wang, X.}, \bibinfo{author}{Tang, L.}, \bibinfo{year}{2017}.
\newblock \bibinfo{title}{A machine-learning based memetic algorithm for the multi-objective permutation flowshop scheduling problem}.
\newblock \bibinfo{journal}{Computers \& Operations Research} \bibinfo{volume}{79}, \bibinfo{pages}{60--77}.
\newblock \DOIprefix\doi{https://doi.org/10.1016/j.cor.2016.10.003}.
\bibitem[{Wang et~al.(2023)Wang, Cai, Li, Yang, Zhao and Xie}]{Wang2023}
\bibinfo{author}{Wang, Z.}, \bibinfo{author}{Cai, B.}, \bibinfo{author}{Li, J.}, \bibinfo{author}{Yang, D.}, \bibinfo{author}{Zhao, Y.}, \bibinfo{author}{Xie, H.}, \bibinfo{year}{2023}.
\newblock \bibinfo{title}{Solving non-permutation flow-shop scheduling problem via a novel deep reinforcement learning approach}.
\newblock \bibinfo{journal}{Computers \&amp; Operations Research} \bibinfo{volume}{151}, \bibinfo{pages}{106095}.
\newblock \URLprefix \url{http://dx.doi.org/10.1016/j.cor.2022.106095}, \DOIprefix\doi{10.1016/j.cor.2022.106095}.
\bibitem[{Wu et~al.(2024)Wu, Yan, Guan and Wei}]{Wu2024}
\bibinfo{author}{Wu, X.}, \bibinfo{author}{Yan, X.}, \bibinfo{author}{Guan, D.}, \bibinfo{author}{Wei, M.}, \bibinfo{year}{2024}.
\newblock \bibinfo{title}{A deep reinforcement learning model for dynamic job-shop scheduling problem with uncertain processing time}.
\newblock \bibinfo{journal}{Engineering Applications of Artificial Intelligence} \bibinfo{volume}{131}, \bibinfo{pages}{107790}.
\newblock \DOIprefix\doi{https://doi.org/10.1016/j.engappai.2023.107790}.
\bibitem[{Xu et~al.(2020)Xu, Hu, Cao, Huang, Liu, Liu, Chen and Blaabjerg}]{Xu2020}
\bibinfo{author}{Xu, X.}, \bibinfo{author}{Hu, W.}, \bibinfo{author}{Cao, D.}, \bibinfo{author}{Huang, Q.}, \bibinfo{author}{Liu, Z.}, \bibinfo{author}{Liu, W.}, \bibinfo{author}{Chen, Z.}, \bibinfo{author}{Blaabjerg, F.}, \bibinfo{year}{2020}.
\newblock \bibinfo{title}{Scheduling of wind-battery hybrid system in the electricity market using distributionally robust optimization}.
\newblock \bibinfo{journal}{Renewable Energy} \bibinfo{volume}{156}, \bibinfo{pages}{47--56}.
\newblock \DOIprefix\doi{https://doi.org/10.1016/j.renene.2020.04.057}.
\bibitem[{Yang et~al.(2022a)Yang, Feng and Guan}]{YangFengGuan2022}
\bibinfo{author}{Yang, S.}, \bibinfo{author}{Feng, M.}, \bibinfo{author}{Guan, D.}, \bibinfo{year}{2022}a.
\newblock \bibinfo{title}{Intelligent scheduling system for production line automatic matching based on dssm-xgboost}.
\newblock \bibinfo{journal}{Journal of Physics: Conference Series} \bibinfo{volume}{2203}, \bibinfo{pages}{012072}.
\newblock \DOIprefix\doi{10.1088/1742-6596/2203/1/012072}.
\bibitem[{Yang et~al.(2022b)Yang, Zhang and Yang}]{YangZhangYang2022}
\bibinfo{author}{Yang, Y.}, \bibinfo{author}{Zhang, X.}, \bibinfo{author}{Yang, L.}, \bibinfo{year}{2022}b.
\newblock \bibinfo{title}{Data-driven power system small-signal stability assessment and correction control model based on xgboost}.
\newblock \bibinfo{journal}{Energy Reports} \bibinfo{volume}{8}, \bibinfo{pages}{710--717}.
\newblock \DOIprefix\doi{https://doi.org/10.1016/j.egyr.2022.02.249}. \bibinfo{note}{iCPE 2021 - The 2nd International Conference on Power Engineering}.
\bibitem[{Yuan et~al.(2024)Yuan, Wang, Cheng, Song, Fan and Li}]{YuanWang2024}
\bibinfo{author}{Yuan, E.}, \bibinfo{author}{Wang, L.}, \bibinfo{author}{Cheng, S.}, \bibinfo{author}{Song, S.}, \bibinfo{author}{Fan, W.}, \bibinfo{author}{Li, Y.}, \bibinfo{year}{2024}.
\newblock \bibinfo{title}{Solving flexible job shop scheduling problems via deep reinforcement learning}.
\newblock \bibinfo{journal}{Expert Systems with Applications} \bibinfo{volume}{245}, \bibinfo{pages}{123019}.
\newblock \DOIprefix\doi{https://doi.org/10.1016/j.eswa.2023.123019}.
\bibitem[{Yuan(2017)}]{Yuan2017}
\bibinfo{author}{Yuan, J.}, \bibinfo{year}{2017}.
\newblock \bibinfo{title}{Unary {NP}-hardness of minimizing the number of tardy jobs with deadlines}.
\newblock \bibinfo{journal}{J. Sched.} \bibinfo{volume}{20}, \bibinfo{pages}{211--218}.
\bibitem[{Zhang et~al.(2012)Zhang, Zheng, Li, Wang, Zhong and Hu}]{Zhang2012}
\bibinfo{author}{Zhang, Z.}, \bibinfo{author}{Zheng, L.}, \bibinfo{author}{Li, N.}, \bibinfo{author}{Wang, W.}, \bibinfo{author}{Zhong, S.}, \bibinfo{author}{Hu, K.}, \bibinfo{year}{2012}.
\newblock \bibinfo{title}{Minimizing mean weighted tardiness in unrelated parallel machine scheduling with reinforcement learning}.
\newblock \bibinfo{journal}{Computers \& Operations Research} \bibinfo{volume}{39}, \bibinfo{pages}{1315--1324}.
\newblock \DOIprefix\doi{https://doi.org/10.1016/j.cor.2011.07.019}.

\end{thebibliography}

\end{document}